\documentclass[runningheads]{llncs}
\usepackage{graphicx}

\usepackage{tikz}
\usepackage{comment}
\usepackage{amsmath,amssymb} 
\usepackage{xcolor}
\usepackage{wrapfig}

\usepackage{optidef}
\usepackage{algorithm}
\usepackage[noend]{algorithmic}
\usepackage{subfigure}
\usepackage{appendix}
\usepackage{bm}
\usepackage{amsfonts}
\usepackage{xcolor}
\usepackage{soul}

\usepackage{amsthm}
\usepackage{pifont}
\usepackage{graphics}

\usepackage{array}
\newcolumntype{C}[1]{>{\centering\let\newline\\\arraybackslash\hspace{0pt}}m{#1}}
\DeclareMathOperator*{\argmin}{arg\,min}
\newtheorem{thm}{Theorem}

\newtheorem{Lem}{Lemma}

\usepackage[accsupp]{axessibility}  

\begin{document}
\pagestyle{headings}
\mainmatter
\def\ECCVSubNumber{5812}  

\title{UniCR: Universally Approximated Certified Robustness via Randomized Smoothing} 


\titlerunning{UniCR}
%

\author{Hanbin Hong\inst{1} \and
Binghui Wang\inst{2} \and
Yuan Hong\inst{1}}
\authorrunning{Hanbin Hong et al.}
%
\institute{University of Connecticut, Storrs CT 06269, USA\\ 
\email{hanbin.hong@uconn.edu},\ \email{yuan.hong@uconn.edu}
\and
Illinois Institute of Technology, Chicago IL 60616, USA \\
\email{bwang70@iit.edu}}

\maketitle
\begin{abstract}
We study certified robustness of machine learning classifiers against adversarial perturbations. 
In particular, we propose the first universally approximated certified robustness (UniCR) framework, which can approximate the robustness certification of \emph{any} input on \emph{any} classifier against \emph{any}  $\ell_p$ perturbations with noise generated by \emph{any} continuous probability distribution. Compared with the state-of-the-art certified defenses, UniCR provides many significant benefits: (1) the first universal robustness certification framework for the above 4 “any”s; 
(2) automatic robustness certification that avoids case-by-case analysis, (3) tightness validation of certified robustness, and (4) optimality validation of noise distributions used by randomized smoothing.
We conduct extensive experiments to validate the above benefits of UniCR and the advantages of UniCR over state-of-the-art certified defenses against $\ell_p$ perturbations. 
\end{abstract}

\section{Introduction}
\label{sec:intro}
\vspace{-2mm}

Machine learning (ML) classifiers are vulnerable to adversarial perturbations~\cite{madry2018towards,carlini2017towards,chen2017zoo,chen2020hopskipjumpattack}). Certified defenses 
\cite{wong2018provable,katz2017reluplex,carlini2017provably,fischetti2018deep,gouk2021regularisation,raghunathan2018certified,croce2019provable,mirman2018differentiable} were recently proposed to ensure provable robustness against adversarial perturbations.
Typically, certified defenses aim to derive a certified radius such that an arbitrary $\ell_p$ (e.g., $\ell_1$, $\ell_2$ or $\ell_\infty$) perturbation, when added to a testing input, cannot fool the classifier, if the $\ell_p$-norm value of the perturbation does not exceed the radius. 
Among all certified defenses, randomized smoothing \cite{li2020second,lecuyer2019certified,cohen2019certified} based certified defense
has achieved the state-of-the-art certified radius and  
can be applied to 
\emph{any} classifier. 
Specifically, given a testing input and any classifier, randomized smoothing first defines a noise distribution and adds sampled noises to the testing input; then builds a smoothed classifier based on the noisy inputs, and finally derives certified radius for the smoothed classified, e.g., using the Neyman-Pearson Lemma~\cite{cohen2019certified}, against an $\ell_p$ perturbation. 

However, existing randomized smoothing based (and actually all) certified defenses 
only focus on specific settings 
and cannot universally certify a classifier against \emph{any} $\ell_p$ perturbation or \emph{any} noise distribution. 
For example, the certified radius derived by Cohen et al. \cite{cohen2019certified} is tied to the Gaussian noise and $\ell_2$ perturbation. 
Recent works~\cite{yang2020randomized,zhang2020black,croce2019provable} propose methods to certify the robustness for multiple norms/noises, e.g., Yang et al. \cite{yang2020randomized} propose the level set and differential method 
to derive the certified radii for multiple noise distributions. However, the  certified radius derivation for different norms is  still \emph{subject to case-by-case theoretical analyses}. 
These methods, although achieving somewhat generalized certified robustness, are still 
lack of universality (See Table \ref{table:related work} for the summary).

\begin{table*}[t]
	\small
	\centering
	\caption{Comparison with highly-related works.}
		\vspace{-0.1in}
		\resizebox{\linewidth}{!}{
        \begin{tabular}{|ccccccc|}	
			\hline
			 &Classifier   &Smoothing Noise & Perturbations &Tightness & Optimizable  & Analysis-free\\
            \hline
            Lecuyer et al. \cite{lecuyer2019certified} &Any  & Gaussian/Laplace & Any $\ell_p, p\in \mathbb{R}^+$ & Loose & No &No\\
            Cohen et al. \cite{cohen2019certified} & Any & Gaussian & $\ell_2$  &Strictly Tight & No  & No\\
			Teng et al. \cite{teng2020ell}  &Any & Laplace  & $\ell_1$  &Strictly Tight & No & No\\
            Dvijotham et al. \cite{DBLP:conf/iclr/DvijothamHBKQGX20}   &Any & f-divergence-constrained   & Any $\ell_p, p\in \mathbb{R}^+$ & Loose & No & No\\
            Croce et al. \cite{croce2019provable}   & ReLU-based & No & Any $\ell_p$ for $p>=1$ & Loose & No & No\\
			Yang et al. \cite{yang2020randomized}    &Any & Multiple types & Any $\ell_p, p\in \mathbb{R}^+$ &Strictly Tight & No & No\\
			Zhang et al. \cite{zhang2020black}   &Any &$\ell_p$-term-constrained & $\ell_1, \ell_2, \ell_\infty$  & Strictly Tight & No & Yes\\
			\hline
			Ours (UniCR)   &Any &Any continuous PDF & Any $\ell_p, p\in \mathbb{R}^+$  & Approx. Tight & Yes & Yes\\
			\hline
		\end{tabular}}
		\vspace{-4mm}
	\label{table:related work}
\end{table*}

In this paper, we develop the first 
\underline{Uni}versally Approximated  \underline{C}ertified \underline{R}obust\\
ness (UniCR) framework based on \emph{randomized smoothing}.
Our framework can automate 
the robustness certification for any input on any classifier against any $\ell_p$ perturbation with noises generated by any \emph{continuous probability density function (PDF)}. As shown in Figure \ref{fig:overview}, our UniCR framework provides four unique significant benefits to make certified robustness more universal, practical and easy-to-use with the above four ``any''s. Our key contributions are as follows:
\begin{figure}[!h]
    \vspace{-6mm}
    \centering
    \includegraphics[width=120mm]{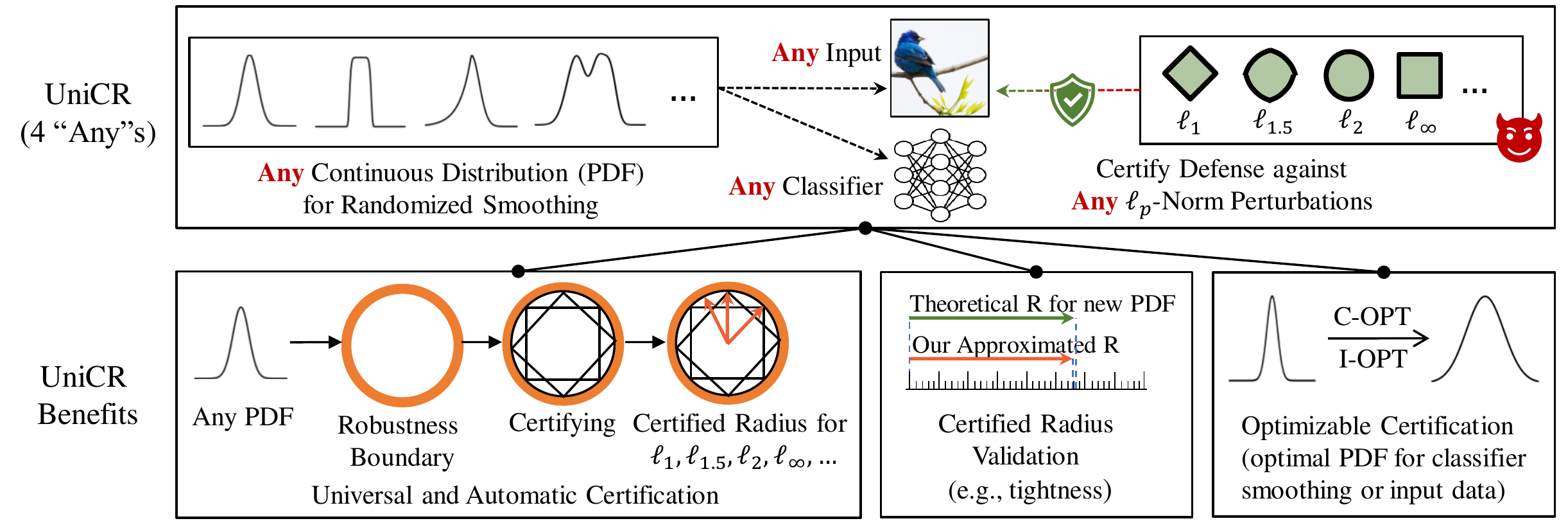}
    \vspace{-3mm}
    \caption{\textmd{Our Universally Approximated Certified Robustness (UniCR) framework.}}
    \label{fig:overview}
    \vspace{-2mm}
\end{figure}


\begin{enumerate}
\vspace{-6mm}
   \item \textbf{Universal Certification}. UniCR is the first universal robustness certification framework for the 4 ``any"s.
   
   
   
    \item \textbf{Automatic Certification}. UniCR 
    provides an automatic robustness certification for all cases. 
    It is easy-to-launch and avoids case-by-case analysis.
   
    \item \textbf{Tightness Validation of Certified Radius}. It is also the first framework that can validate the \emph{tightness} of the derived certified radius in existing certification methods \cite{li2020second,lecuyer2019certified,cohen2019certified} or future methods based on any continuous noise PDF. In Section \ref{sec:rbstudy}, we validate the tightness of the state-of-the-art certification methods (e.g., see Figure \ref{fig:P-R curve}).



    
    \item \textbf{Optimality Validation of Noise PDFs}.
    UniCR can also automatically tune the parameters in noise PDFs to strengthen the robustness certification against any $\ell_p$ perturbations. For instance, On CIFAR10 and ImageNet datasets, UniCR improves as high as $38.78\%$ overall performance over the state-of-the-art certified defenses against all $\ell_p$ perturbations. In Section \ref{sec:exp}, we show that Gaussian noise and Laplace noise are not the optimal randomization distribution against the $\ell_2$ and $\ell_1$ perturbation, respectively.
    
    
    

\end{enumerate}

\section{Universally Approximated Certified
Robustness}
\label{sec:model}
\vspace{-2mm}

In this section, we propose the theoretical foundation for universally certifying a testing input against any $\ell_p$ perturbations with noise from any continuous PDF. 

\vspace{-2mm}
\subsection{Universal Certified Robustness}

Consider a general classification problem that classifies input data in $\mathbb{R}^d$ to a class belonging to a set of classes $\mathcal{Y}$. 
Given an input $x \in \mathbb{R}^d$, an \emph{any} (base) classifier $f$ that maps $x$ to a class in $\mathcal{Y}$, and a random noise $\epsilon$ from \emph{any} continuous PDF $\mu_x$. 
We define a 
smoothed classifier $g$  as the most probable class over the noise-perturbed input:
\begin{equation}
    g(x)=\arg \max_{c\in \mathcal{Y}} \mathbb{P}(f(x+\epsilon)=c)
\label{general randomized smoothing}
\end{equation}
Then, we show that the input has a certified accurate prediction against any $l_p$ perturbation and its certified radius is given by the following theorem.

\begin{thm}{(\textbf{Universal Certified Robustness})}
Let $f:\mathbb{R}^d \rightarrow \mathcal{Y}$ be any deterministic or random classifier, and let $\epsilon$ be drawn from an arbitrary continuous PDF $\mu_x$. 
Denote $g$ as the smoothed classifier in Equation (\ref{general randomized smoothing}), the most probable and second probable classes for predicting a testing input $x$ via $g$ as $c_A, c_B \in \mathcal{Y}$, respectively. If the lower bound of the class $c_A$'s prediction probability $\underline{p_A} \in [0, 1]$, and the upper bound of the class $c_B$'s prediction probability $\overline{p_B}\in [0,1] $ satisfy:
{
\begin{equation}
    \mathbb{P}(f(x+\epsilon)=c_A) \ge \underline{p_A} \ge \overline{p_B} \ge \max_{c \neq c_A} \mathbb{P}(f(x+\epsilon)=c)
\label{eq2}
\end{equation}
}%
Then, we guarantee that $g(x+\delta)=c_A$ for all $||\delta||_p \leq R$,  
where $R$ is called the {\bf certified radius} and it is the minimum $\ell_p$-norm of all the adversarial perturbations $\delta$ that satisfies the {\bf robustness boundary conditions} as below:
\vspace{-2mm}
{\small
\begin{align}
    &\mathbb{P}(\frac{\mu_x(x-\delta)}{\mu_x(x)}\leq t_A)=\underline{p_A}, \quad 
    \mathbb{P}(\frac{\mu_x(x-\delta)}{\mu_x(x)}\ge t_B)=\overline{p_B}, \nonumber\\
    &\mathbb{P}(\frac{\mu_x(x)}{\mu_x(x+\delta)}\leq t_A) = \mathbb{P}(\frac{\mu_x(x)}{\mu_x(x+\delta)}\ge t_B) \label{thm2: conditions}
\end{align}
}%
where 
$t_A$ and $t_B$ are auxiliary parameters to satisfy the above conditions.
\label{thm2}
\end{thm}

\begin{proof}
See the detailed proof in Appendix \ref{apd: thm2 proof}. 
\end{proof}

\noindent\textbf{Robustness Boundary}. Theorem \ref{thm2} provides a novel insight that meeting certain conditions is equivalent to deriving the certified robustness. 
The conditions in Equation (\ref{thm2: conditions}) construct a boundary in the perturbation $\delta$ space, which is defined as the ``\emph{robustness boundary}''. Within this robustness boundary, the  prediction outputted by the smoothed classifier $g$  is certified to be consistent and correct. The robustness boundary, rather than the certified radius, is actually more general to measure the certified robustness since the space constructed by each certified radius (against any specific $\ell_p$ perturbation) is only a subset of the space inside the robustness boundary. Without loss of generality, the traditional certified radius can be alternatively defined as the radius that maximizes the $\ell_p$ ball inside the robustness boundary, which is also the perturbation $\delta$ on the boundary that minimizes $||\delta||_p$. 
Figure \ref{fig:insight} illustrates the relationship between certified radius and the robustness boundary against $\ell_1$, $\ell_2$ and $\ell_\infty$ perturbations. 

Notice that, given any continuous noise PDF, the corresponding robustness boundary for all the $\ell_p$-norms would naturally exist. Each maximum $\ell_p$ ball is a subspace of the robustness boundary, and gives the certified radius for that specific $\ell_p$-norm. Thus, all the certified radii can be universally derived, and Theorem \ref{thm2} provides a theoretical foundation to certify any input against any $\ell_p$ perturbations with any continuous noise PDF.

\begin{figure}[!t]
    \centering
    \includegraphics[width=85mm]{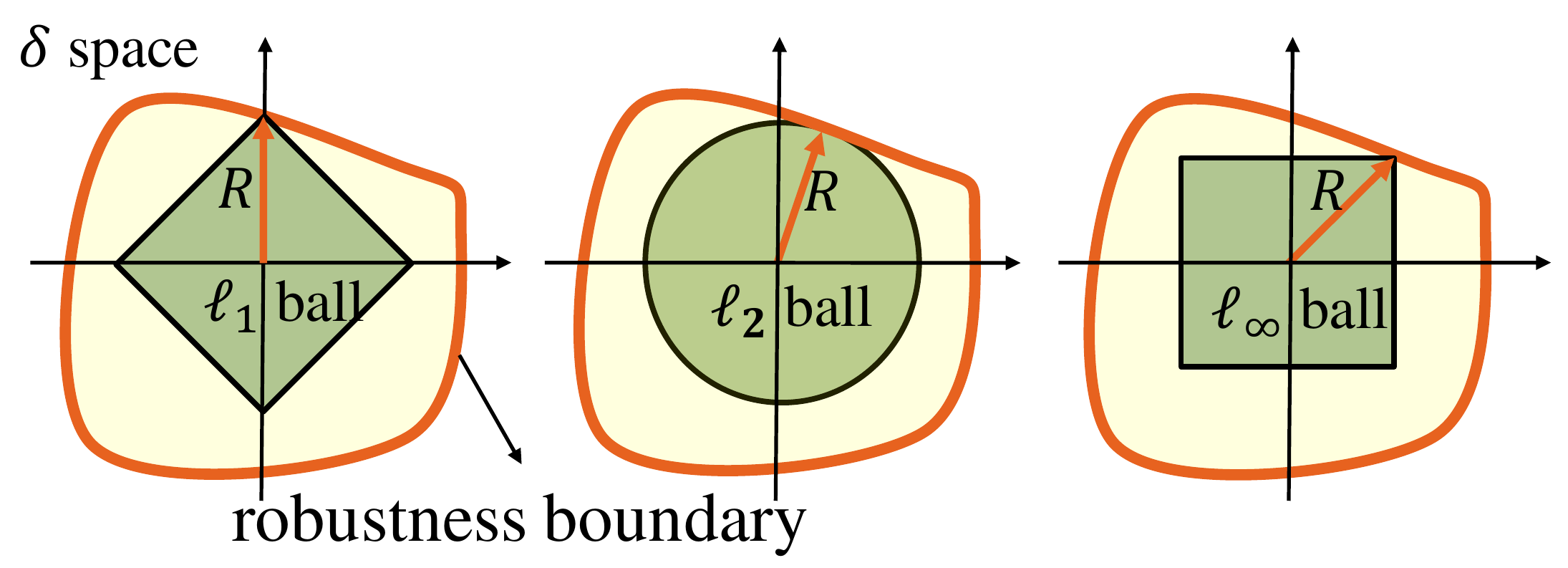}
    \vspace{-3mm}
    \caption{\textmd{An illustration to Theorem \ref{thm2}. The conditions in Theorem \ref{thm2} construct a ``{\color{red}\textbf{Robustness Boundary}}" in $\delta$ space. In case of a perturbation inside the robustness boundary, the smoothed prediction can be certifiably correct.  From left to right, the figures show that the minimum $||\delta||_1$, $||\delta||_2$ and $||\delta||_\infty$ on the robustness boundary are exactly the certified radius $R$ in $\ell_1$, $\ell_2$ and $\ell_\infty$-norm, respectively.}} 
    \vspace{-6mm}
    \label{fig:insight}
\end{figure}


\vspace{0.05in}
\noindent\textbf{All $\ell_p$ Perturbations}. 
Although we mainly introduce UniCR against  $\ell_1$, $\ell_2$ and $\ell_\infty$ 
perturbations, our UniCR is not limited to these three norms. 
We emphasize that any $p\in \mathbb{R}^+$ (See Appendix \ref{sec: dif p}) 
can be used 
and our UniCR can derive the corresponding certified radius since our robustness boundary gives a general boundary in the $\delta$ perturbation space.

\vspace{-3mm}
\subsection{Approximating Tight Certified Robustness}
\label{sec: UCR framework}



The tight certified radius can be derived by finding a perturbation $\delta$ on the robustness boundary that has a minimum $\|\delta\|_p$
(for any $p \in \mathbb{R}^+$). 
However, it is challenging to either 
find a perturbation $\delta$ that is exactly on the robustness boundary, or find the minimum $||\delta||_p$. Here, we design an alternative two-phase optimization scheme to accurately approximate the tight certification in practice.
In particular, Phase I is to suffice the conditions such that $\delta$ is on the robustness boundary, and Phase II is to minimize the $\ell_p$-norm.

 We perform Phase I by the ``scalar optimization'', where any perturbation $\delta$ will be $\lambda$-scaled to the robustness boundary (see \textcircled{1} in Figure \ref{fig:ucr_opt}). We perform Phase II by the ``direction optimization'', where the direction of $\delta$ will be optimized towards a minimum $\|\lambda  \delta\|_p$ (see \textcircled{2} in Figure \ref{fig:ucr_opt}). In the two-phase optimization, the direction optimization will be iteratively executed until finding the minimum $||\lambda \delta||_p$, where the perturbation $\delta$  will be scaled to the robustness boundary beforehand in every iteration. 
 Thus, the intractable optimization problem in Equation~\ref{thm2: conditions} can be converted to:
{\small
\begin{align}
    &R=||\lambda \delta||_p, 
    \nonumber\\
    & s.t.~~~~  \delta \in \argmin_{\delta} ||\lambda \delta||_p, \quad 
    \lambda = \argmin_{\lambda} |K| 
    , \nonumber \\ 
    &\qquad{\mathbb{P}(\frac{\mu_x(x-\lambda\delta)}{\mu_x(x)}\leq t_A)=\underline{p_A}}, \quad 
    \mathbb{P}(\frac{\mu_x(x-\lambda\delta)}{\mu_x(x)}\ge t_B)=\overline{p_B}, \nonumber\\
    &\qquad K=\mathbb{P}(\frac{\mu_x(x)}{\mu_x(x+\lambda\delta)}\leq t_A) 
    -\mathbb{P}(\frac{\mu_x(x)}{\mu_x(x+\lambda\delta)}\ge t_B).
\label{opt 2}
\vspace{-0.2in}
\end{align}
}%
The scalar optimization in Equation (\ref{opt 2}) aims to find the scale factor $\lambda$ that scales a perturbation $\delta$ to the boundary so that $|K|$ approaches $0$. With the scalar $\lambda$ for ensuring that the scaled $\delta$ is nearly on the boundary, the direction optimization optimizes the perturbation $\delta$'s direction to find the certified radius $R=||\lambda \delta||_p$. We also present the theoretical analysis on the certification confidence and the optimization convergence in Appendix \ref{apd: confidence} and \ref{apd:convergence}, respectively.







\begin{figure}[!t]
    \centering
    \includegraphics[width=85mm]{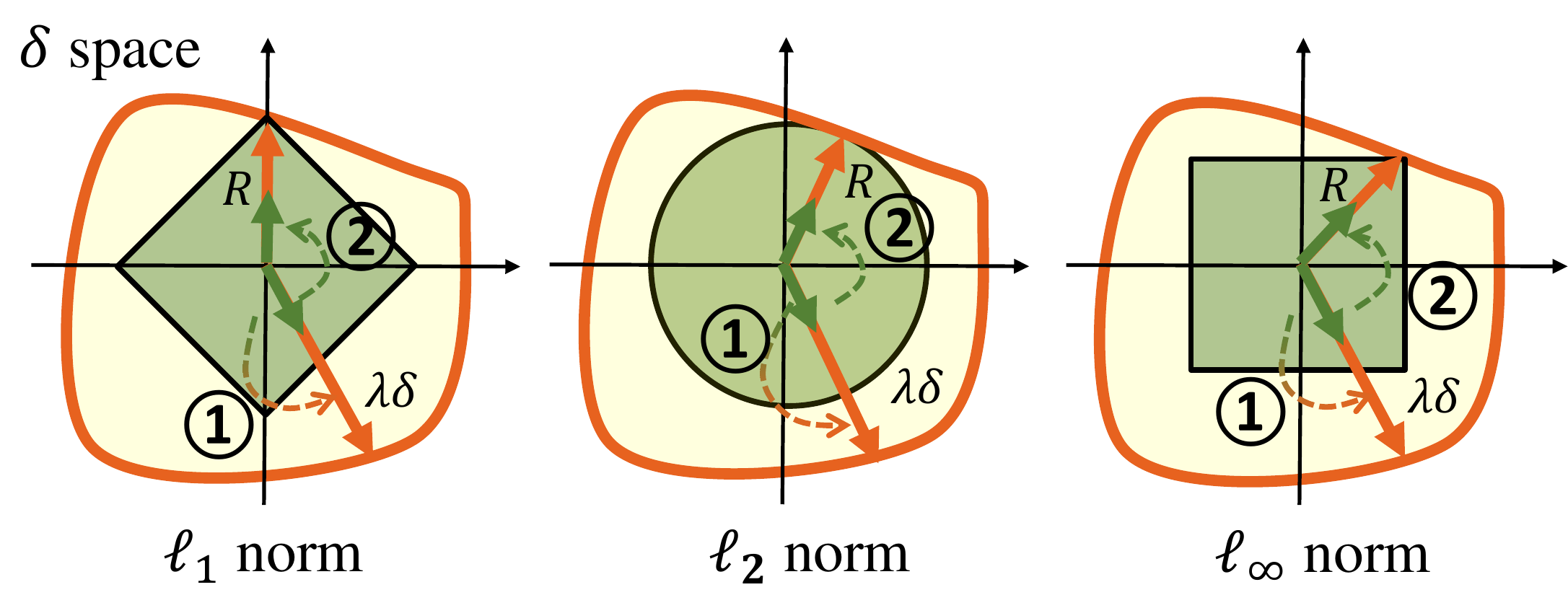}
    \centering
    \vspace{-3mm}
    \caption{\textmd{An illustration to estimating the certified radius. The scalar optimization (\textcircled{1}) and direction optimization (\textcircled{2}) effectively find the minimum $||\delta||_p$ within the robustness boundary, which is the certified radius $R$. 
    }}
    \vspace{-4mm}
    \label{fig:ucr_opt}
\end{figure}

\section{Deriving Certified Radius within Robustness Boundary}
\label{sec:rbstudy}
\vspace{-2mm}
In this section, we will introduce how to universally and automatically derive the certified radius against any $\ell_p$ perturbations within the robustness boundary constructed by any noise PDF. In particular, we will present practical algorithms for solving the two-phase optimization problem to approximate the certified radius, empirically validate that our UniCR approximates the tight certified radius derived by recent works \cite{cohen2019certified,yang2020randomized,teng2020ell}, and finally discuss how to apply UniCR to validate the radius of existing certified defenses.

\vspace{-2mm}
\subsection{Calculating Certified Radius in Practice}
\label{sec:two-phase algorithms}
Following the existing randomized smoothing based defenses~\cite{cohen2019certified,teng2020ell}, we first use the Monte Carlo method to estimate the probability bounds ($\underline{p_A}$ and $\overline{p_B}$). Then, we use them in our two-phase optimization scheme to derive the certified radius. 


 \vspace{0.05in}

\noindent\textbf{Estimating Probability Bounds}. The two-phase optimization needs to estimate the probabilities bounds $\underline{p_A}$ and $\overline{p_B}$ and compute two auxiliary parameters $t_A$ and $t_B$ (required by the certified robustness based on the Neyman-Pearson Lemma in Appendix \ref{sec:preli}). Identical to existing works \cite{cohen2019certified,teng2020ell}, the probabilities bounds $\underline{p_A}$ and $\overline{p_B}$ are commonly estimated by the Monte Carlo method \cite{cohen2019certified}. Given the estimated $\underline{p_A}$ and $\overline{p_B}$ as well as any given noise PDF and a perturbation $\delta$, we also use the Monte Carlo method to estimate the cumulative density function (CDF) of fraction $\mu_x(x-\lambda\delta)/\mu_x(x)$. Then, we can compute the auxiliary parameters $t_A$ and $t_B$. Specifically, the auxiliary parameters $t_A$ and $t_B$ can be computed by $t_A=\Phi^{-1}(\underline{p_A})$ and $t_B=\Phi^{-1}(\overline{p_B})$, where $\Phi^{-1}$ is the inverse CDF of the fraction $\mu_x(x-\lambda\delta)/\mu_x(x)$. The procedures for computing $t_A$ and $t_B$ are detailed in Algorithm \ref{alg1} in Appendix \ref{apd: algs}.


\vspace{0.05in}

\noindent\textbf{Scalar Optimization}. Finding a perturbation $\delta$ that is exactly on the robustness boundary is computationally challenging. 
Thus, we alternatively scale the $\delta$ to approach the boundary. We use the binary search algorithm to find a scale factor that minimizes $|K|$ (\emph{the distance between $\delta$ and the robustness boundary}). The algorithm and detailed description are presented in Appendix \ref{apd: scalarAlg}. 




\vspace{0.05in}

\noindent\textbf{Direction Optimization}. We use the Particle Swarm Optimization (PSO) method \cite{kennedy1995particle} to find $\delta$ that minimizes the $\ell_p$-norm after scaling to the robustness boundary. In each iteration of PSO, the particle's position represents $\delta$, and the cost function is $f_{PSO}(\delta)=||\lambda \delta||_p$, where the scalar $\lambda$ is found by the scalar optimization. The PSO aims to find the position $\delta$ that can minimize the cost function. To pursue convergence, we choose some initial positions in symmetry for different $\ell_p$-norms.
Empirical results show that the radius obtained by PSO with these initial positions can accurately approximate the tight certified radius. We show how to set the initial positions in Appendix \ref{apd: DirectAlg}.

In our experiments, the certification (deriving the certified radius) can be efficiently completed on MNIST \cite{lecun2010mnist}, CIFAR10 \cite{krizhevsky2009learning} and ImageNet \cite{ILSVRC15} datasets (less than 10 seconds per image), as shown in Appendix \ref{sec:eff}.




\vspace{0.05in}


\noindent {\bf Certified Radius Comparison with State-of-The-Arts.} We compare the certified radius obtained by our two-phase optimization method and that by the state-of-the-arts~\cite{cohen2019certified,yang2020randomized,teng2020ell} and the comparison results are shown in Figure \ref{fig:P-R curve}.  
Note that the certified radius is a function of $p_A$ (the prediction probability of the top-1 class). The $p_A$-$R$ curve can well depict the certified radius $R$ w.r.t. $p_A$. 
We observe that 
our $p_A$-$R$ curve highly approximates the tight theoretical curves in existing works, e.g., the Gaussian noise against $\ell_2$ and $\ell_\infty$ perturbations \cite{cohen2019certified,yang2020randomized}, Laplace noise against $\ell_1$ perturbations \cite{teng2020ell}, as well as General Normal noise and General Exponential noise derived by Yang et al. \cite{yang2020randomized}'s method. 

\begin{figure}[!t]
\centering
\subfigure[Laplace vs. $\ell_1$ ]{
\begin{minipage}[t]{0.3\linewidth}
\centering
\includegraphics[width=38 mm]{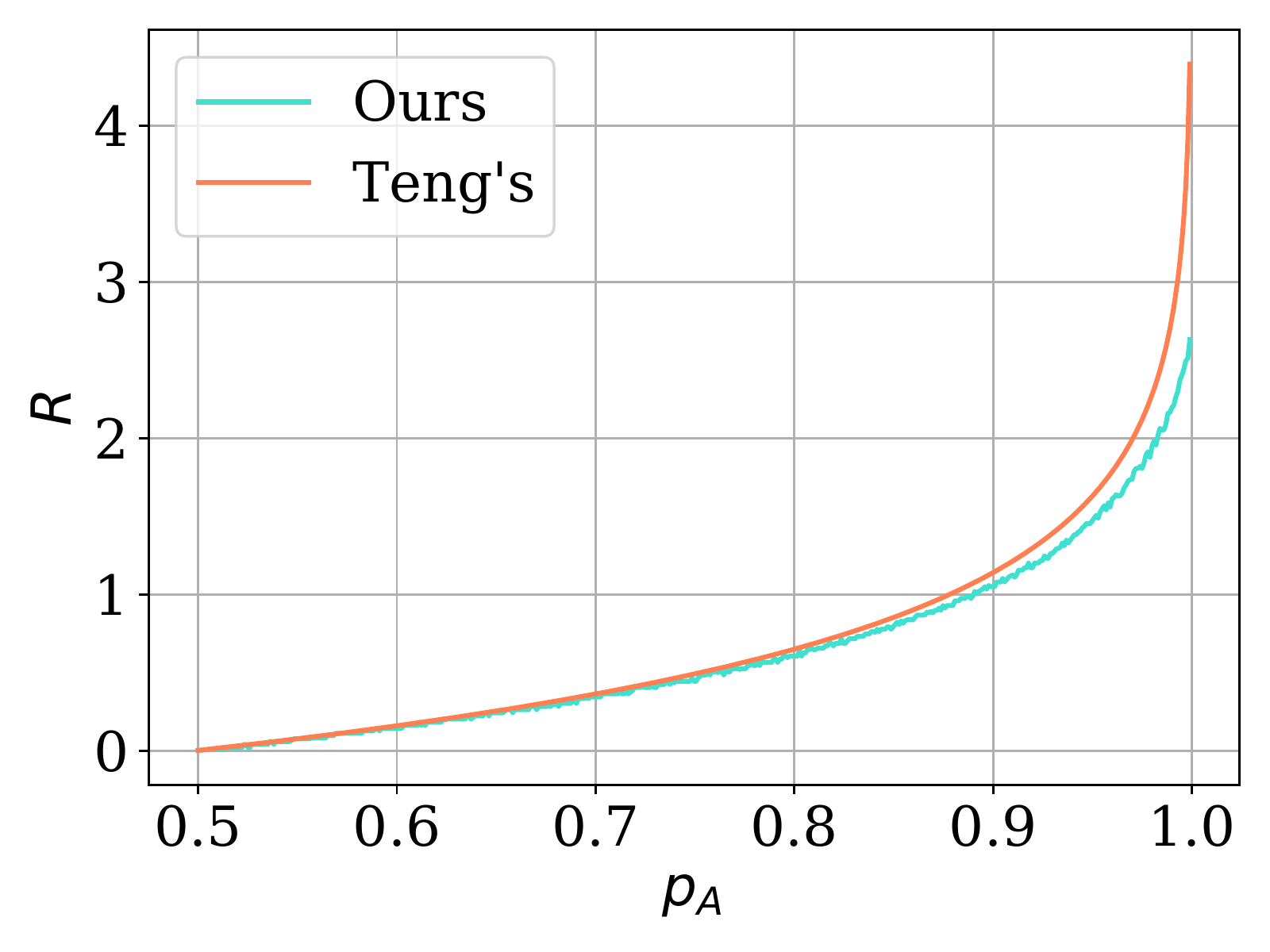}
\end{minipage}
}
\subfigure[ Gaussian vs. $\ell_2$ ]{
\begin{minipage}[t]{0.3\linewidth}
\centering
\includegraphics[width=38 mm]{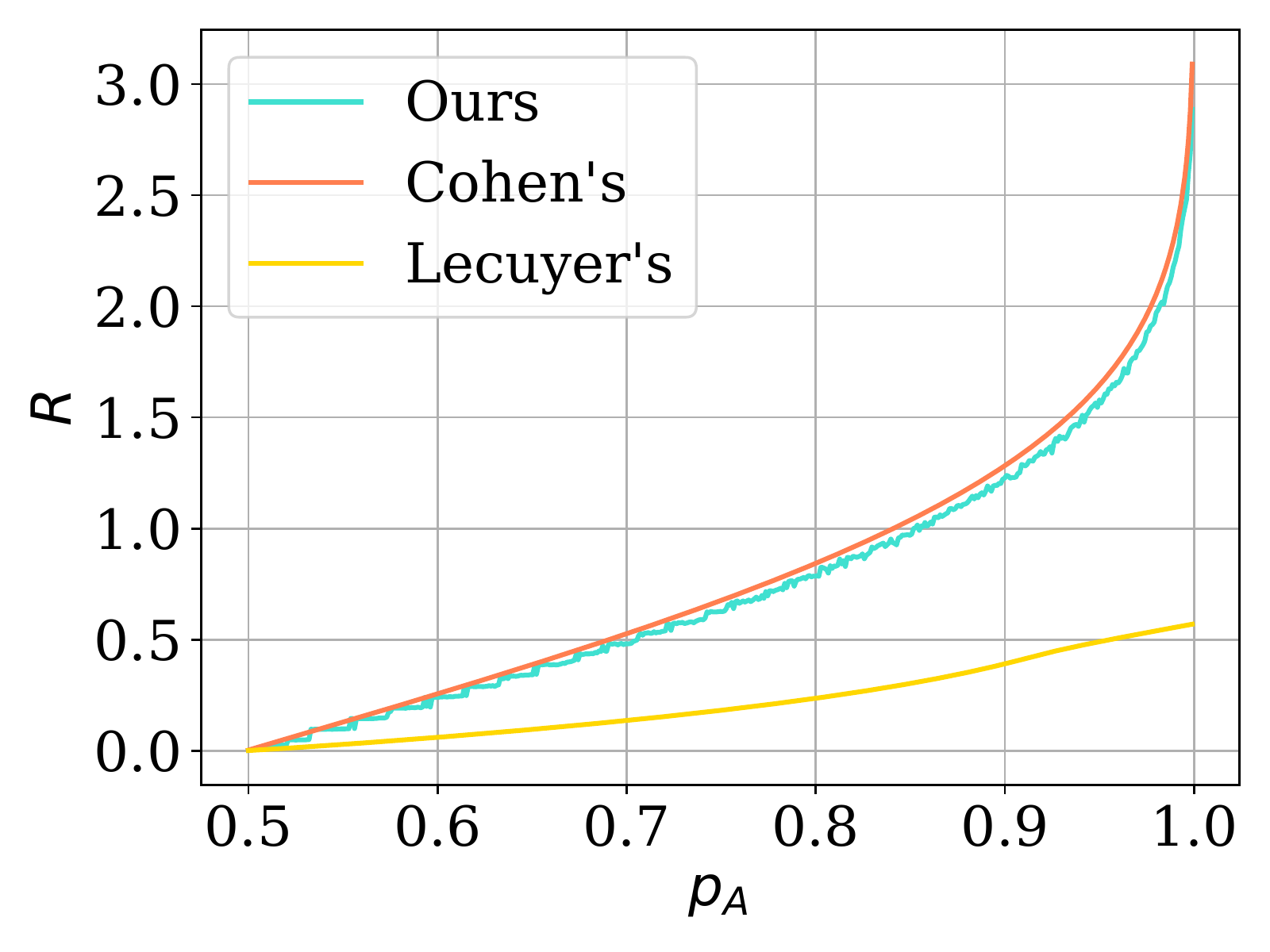}
\end{minipage}
}
\subfigure[Gaussian vs. $\ell_\infty$]{
\begin{minipage}[t]{0.3\linewidth}
\centering
\includegraphics[width=38 mm]{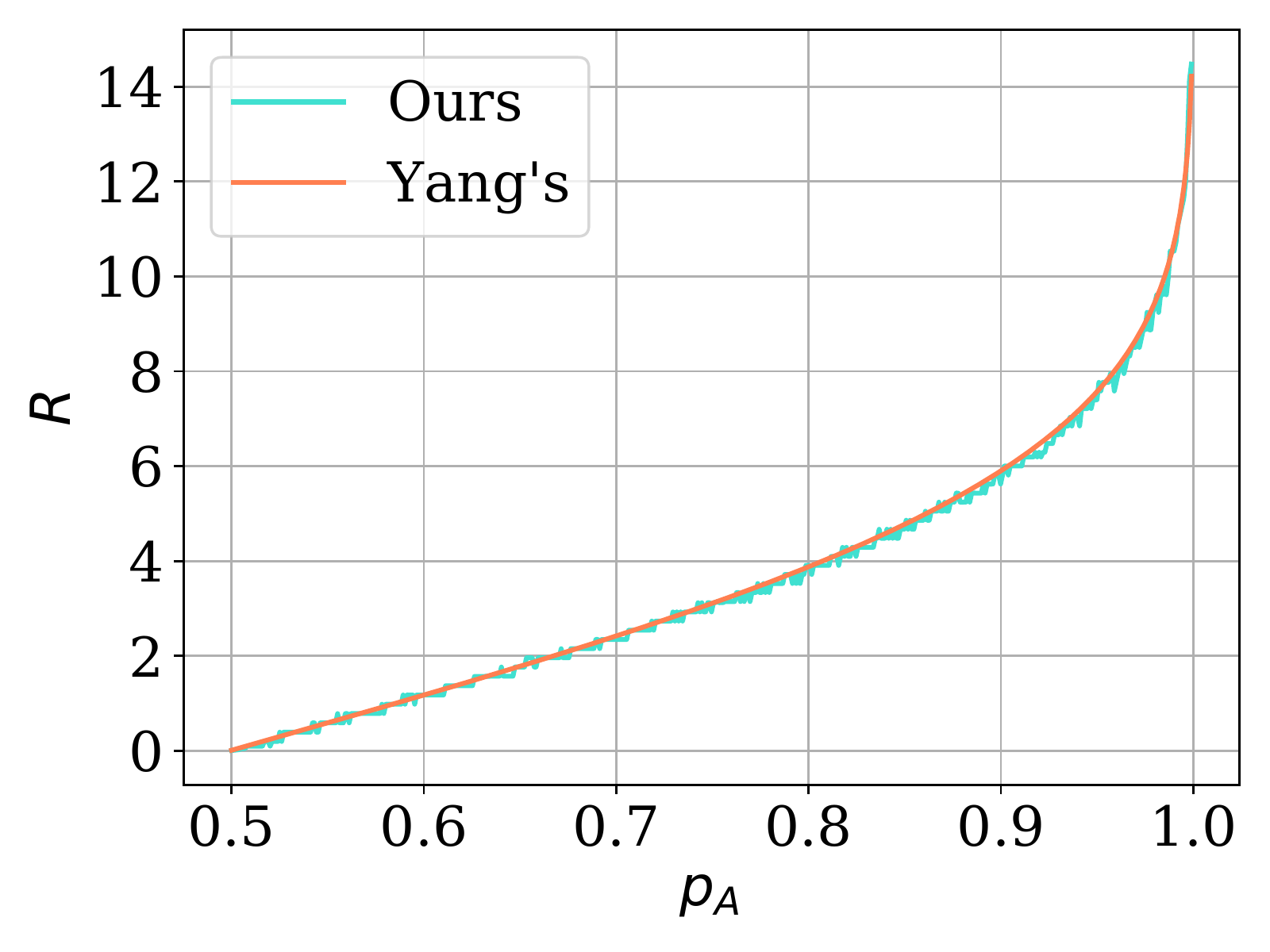}
\end{minipage}
}
\subfigure[Pareto vs. $\ell_1$]{
\begin{minipage}[t]{0.3\linewidth}
\centering
\includegraphics[width=38 mm]{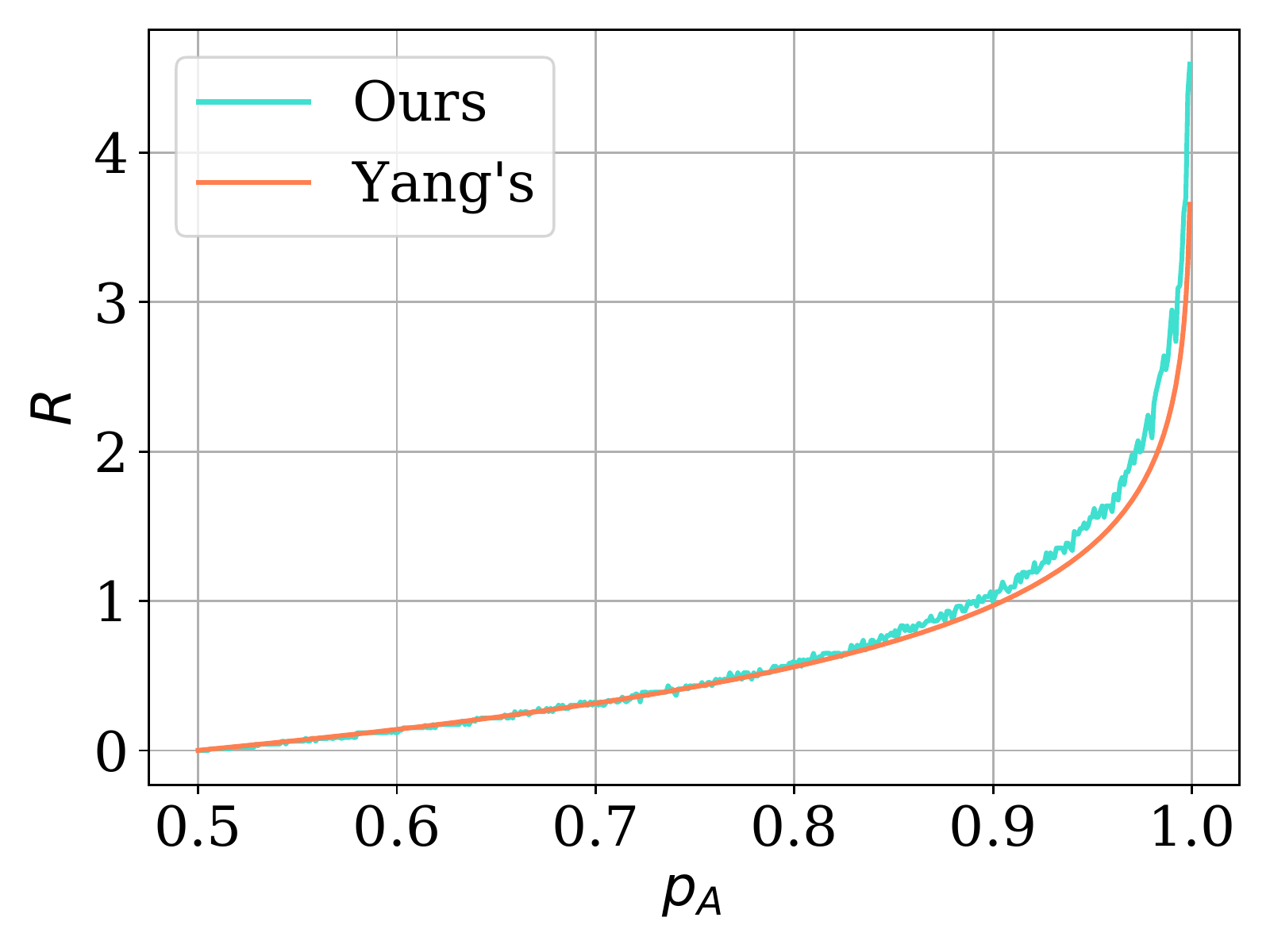}
\end{minipage}
}
\subfigure[General Normal vs. $\ell_1$]{
\begin{minipage}[t]{0.3\linewidth}
\centering
\includegraphics[width=38 mm]{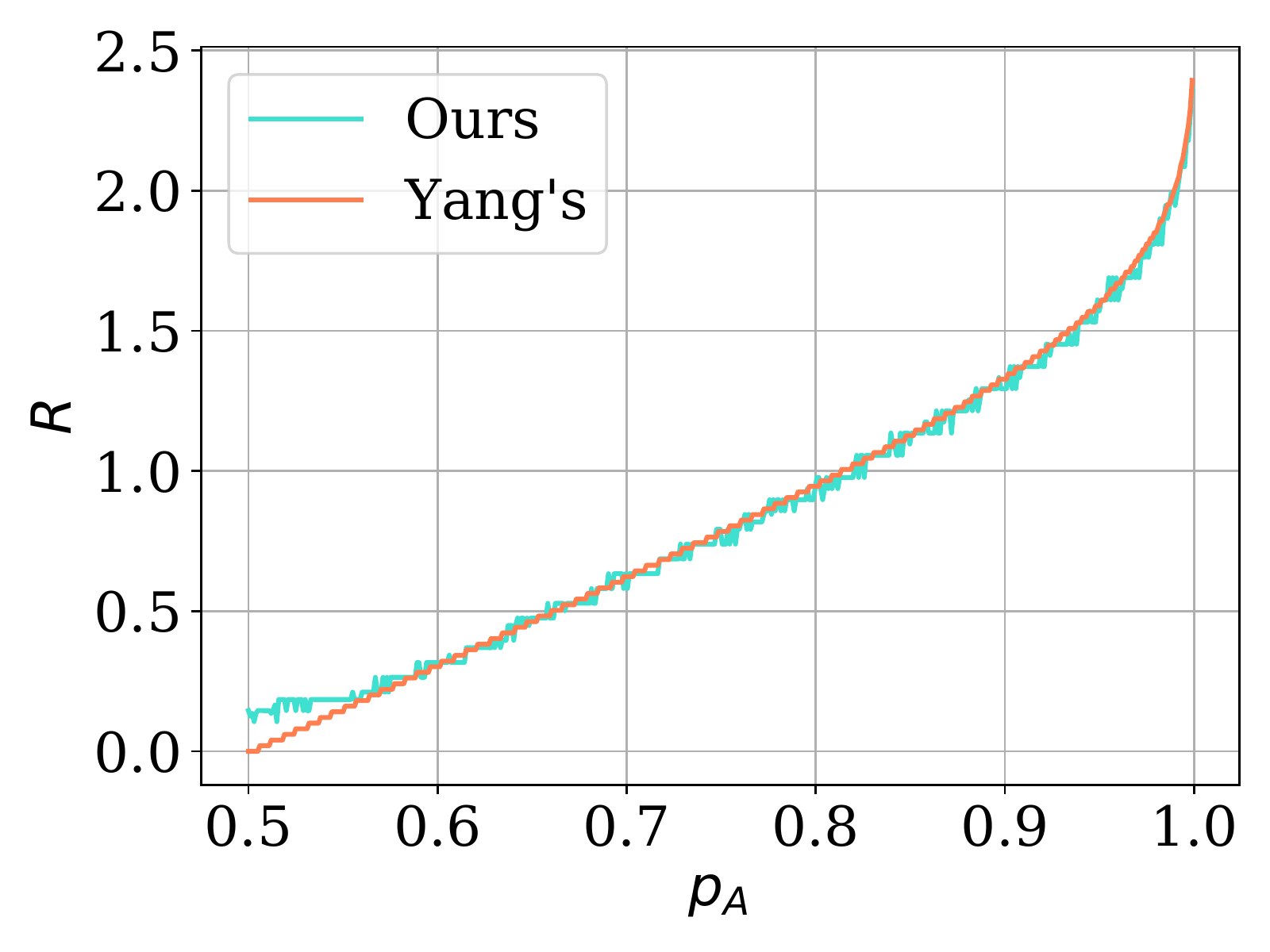}
\end{minipage}
}
\subfigure[Exponential Mix. vs. $\ell_1$]{
\begin{minipage}[t]{0.3\linewidth}
\centering
\includegraphics[width=38 mm]{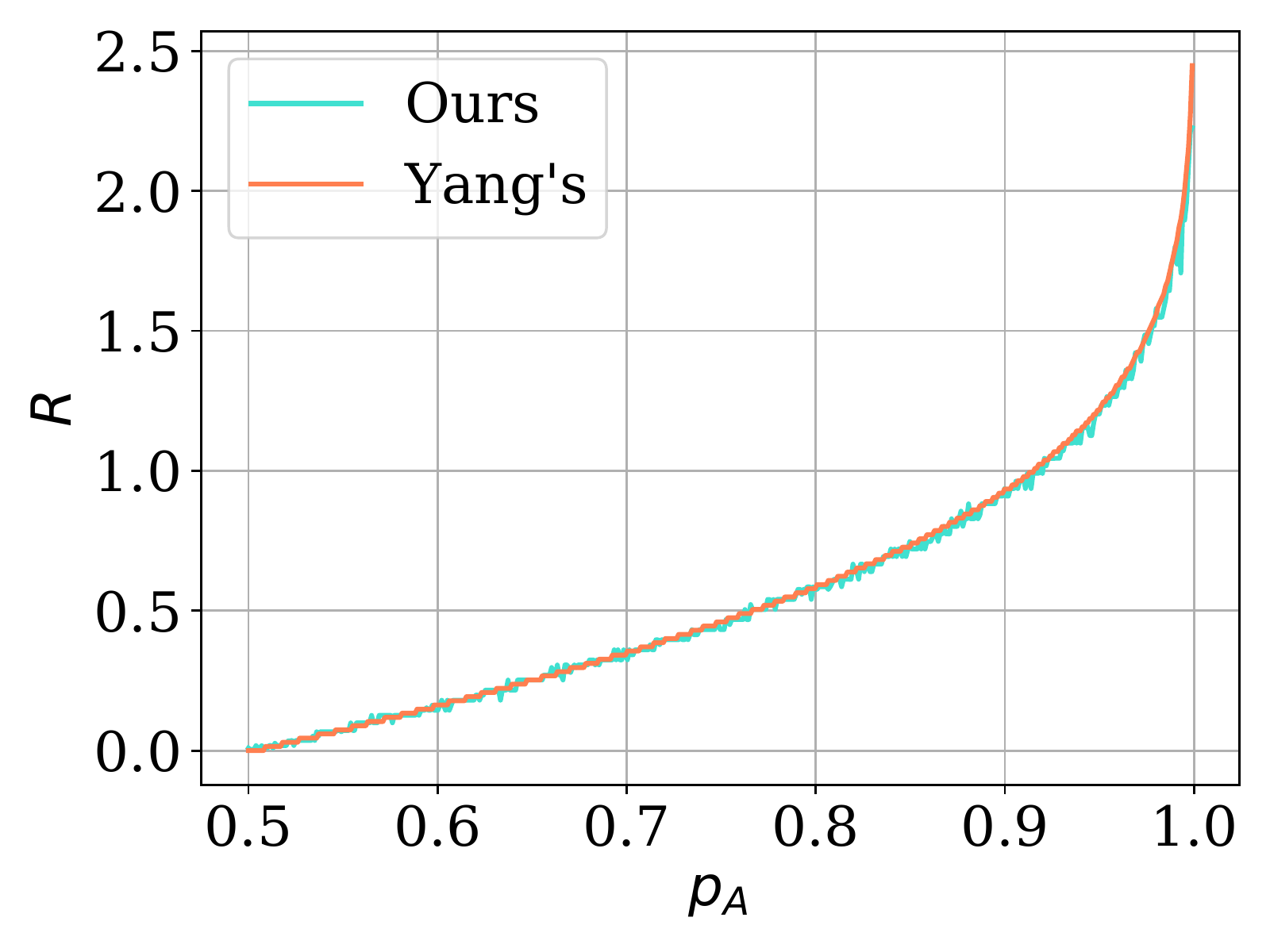}
\vspace{-6mm}
\end{minipage}
}
\centering
\vspace{-3mm}
\caption{$p_A$-$R$ curve comparison of our method and state-of-the-arts (i.e., Teng et al. \cite{teng2020ell}, Cohen et al. \cite{cohen2019certified}, Lecuyer et al. \cite{lecuyer2019certified} and Yang et al. \cite{yang2020randomized}). {We observe that the certified radius obtained by our UniCR is close to that obtained by the state-of-the-arts. These results demonstrate that our UniCR can approximate the tight certification to any input in any $\ell_p$ norm with any continuous noise distribution. 
We also evaluate our UniCR's defense accuracy against a diverse set of attacks, including universal attacks \cite{co2019procedural}, white-box attacks \cite{croce2020reliable,wong2019wasserstein}, and black-box attacks \cite{andriushchenko2020square,chen2020hopskipjumpattack}, and against $\ell_1,\ell_2$ and $\ell_\infty$ perturbations. The experimental results show that UniCR is as robust as the state-of-the-arts ($100\%$ defense accuracy) against all the types of the real attack. The detailed experimental settings and results are presented in Appendix \ref{apd: realattack}.
}}
\label{fig:P-R curve}
\vspace{-3mm}
\end{figure}

\vspace{0.05in}

\noindent\textbf{Tightness Validation of Certified Radius}. Since our UniCR accurately approximates the tight certified radius, it can be used as an auxiliary tool to validate whether an obtained certified radius is tight or not. For example, the certified radius derived by PixelDP \cite{lecuyer2019certified}\footnote{PixelDP \cite{lecuyer2019certified} adopts differential privacy \cite{DworkR14}, e.g., Gaussian mechanism to generate noises for each pixel such that certified robustness can be achieved for images.} is loose, because \cite{lecuyer2019certified}'s $p_A$-$R$ curve in Figure \ref{fig:P-R curve}(b) is far below ours. Also, Yang et al. \cite{yang2020randomized} derives a low bound certified radius for Pareto Noise (Figure \ref{fig:P-R curve}(d))--- It shows that this certified radius is not tight either since it is below ours. For those theoretical radii that are slightly above our radii, they are likely to be tight.

Moreover, due to the high university, our UniCR can even derive the certified radii for complicated noise PDFs, e.g., mixture distribution in which the certified radii are difficult to be theoretically derived. In Section \ref{exp: dif noises}, we show some examples of deriving radii using UniCR on a wide variety of noise distributions in Figure \ref{fig:p-r curve l1}-\ref{fig:p-r curve linf}. In most examples, the certified radii have not been studied before. 





\vspace{-1mm}
\section{Optimizing Noise PDF for Certified Robustness}
\label{sec:optimal}
\vspace{-2mm}
UniCR can derive the certified radius using any continuous noise PDF for randomized smoothing. This provides the flexibility to optimize a noise PDF for enlarging the certified radius.
In this section, we will optimize the noisy PDF in our UniCR framework for obtaining better certified robustness.


\vspace{-3mm}
\subsection{Noise PDF Optimization}
\vspace{-2mm}
All the existing randomized smoothing methods \cite{cohen2019certified,teng2020ell,yang2020randomized,zhang2020black} use the same noise for training the smoothed classifier and certifying the robustness of testing inputs. The motivation is that: the training can improve the lower bound of the prediction probability over the the same noise as the certification. 
Here, the question we ask is: Must we necessarily use the same noise PDF to train the smoothed classifier and derive the certified robustness? Our answer is No!


    


\begin{wrapfigure}{R}{0.6\textwidth}
\vspace{-10mm}
    \centering
    \includegraphics[width=60 mm]{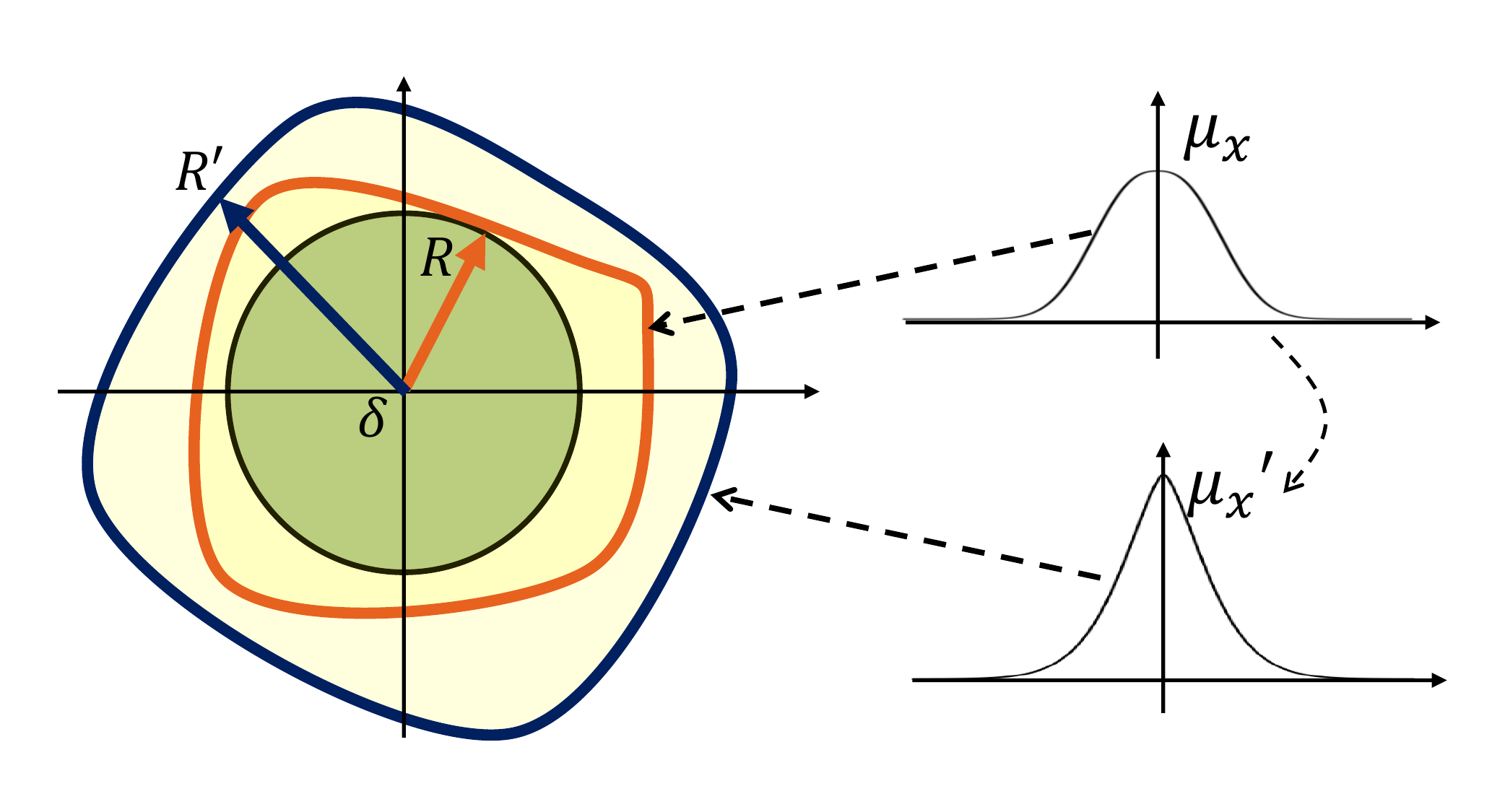}
    \vspace{-4mm}
    \caption{An illustration to noise PDF optimization (take $\ell_2$-norm perturbation as an example). The noise distribution is tuned from $\mu_x$ to $\mu_x'$, which enlarges the robustness boundary. Thus,  UniCR can find a larger certified radius $R'$.}
    \vspace{-8mm}
    \label{fig:optimization}
\end{wrapfigure}

We study the master optimization problem that uses UniCR as a function to maximize the certified radius by tuning the noise PDF (different randomization), as shown in Figure \ref{fig:optimization}. To defend against certain $\ell_p$ perturbations for a classifier, we consider the noise PDF as a variable (Remember that UniCR can provide a certified radius for each noise PDF), and study the following two master optimization problems with two different strategies:


\begin{enumerate}
\vspace{-1mm}

    \item \emph{Classifier-Input Noise Optimization} (``\textbf{C-OPT}''): finding the optimal noise PDF and injecting \emph{the same noise} from this noise PDF into both the training data to train a classifier and testing input to build a smoothed classifier. 
    
    \item \emph{Input Noise Optimization} (``\textbf{I-OPT}''): Training a classifier with the standard noise (e.g., Gaussian noise), while finding the optimal noise PDF for the testing input and injecting noise from this PDF into the testing input only. 
\end{enumerate}



\vspace{-3mm}
\subsection{C-OPT and I-OPT}
\vspace{-2mm}

Before optimizing the certified robustness, we need to define metrics for them. First, since I-OPT only optimizes the noise PDF when certifying each testing input, a ``better'' randomization in I-OPT can be directly indicated by a larger certified radius for a specific input. Second, since C-OPT optimizes the noise PDF for the entire dataset in both training and robustness certification, a new metric for the performance on the entire dataset need to be defined. 

Existing works \cite{zhang2020black,yang2020randomized} draw several certified accuracy vs. certified radius curves computed by noise with different variances (See Figure \ref{fig:Rscore} in Appendix \ref{apd: Metrics}). These curves 
represent the certified accuracy at a range of certified radii, where the certified accuracy at radius $R$ is defined as the percent of the testing samples with a derived certified radius larger than $R$ (and correctly predicted by the smoothed classifier). To simply measure the overall performance, we use the area under the curve 
as an overall metric to the certified robustness, namely ``robustness score''. Then, we design the C-OPT method based on this metric. Specifically, the robustness score $R_{score}$ is formally defined as below:
{\small
\begin{equation}
\label{eq: robustness score}
    R_{score}=\int_{0}^{+\infty} \max_{\sigma}(Acc_{\sigma}(R)) dR, \sigma \in \Sigma,
\end{equation}
}%
where $Acc_{\sigma}(R)$ is the certified accuracy at radius $R$ computed by the noises with variance $\sigma$, and $\Sigma$ is a set of candidate $\sigma$. 

Notice that our UniCR can automatically approximate the certified radius and compute the robustness score w.r.t. different noise PDFs, thus we can tune the noise PDF towards a better robustness score. From the perspective of optimization, denoting the noise PDF as $\mu$, the C-OPT and the I-OPT problems are defined as $\max_\mu R_{score}$ for a classifier and $\max_\mu R$ for an input, respectively.



\vspace{0.05in}

\noindent \textbf{Algorithms for Noise PDF Optimization.} 
We use grid-search in C-OPT to search the best parameters of the noise PDF. We use Hill-Climbing algorithms in I-OPT to find the best parameters of the noise PDF around the noise distribution used in training while maintaining the certified accuracy. 

With the UniCR for estimating the certified radius, we can further tune the noise PDF to each input or the classifier. Specifically, let $\mu(x,\bm{\alpha})$ denote the noise PDF, where $\bm{\alpha}$ is a set of hyper-parameters in the function, i.e., $\bm{\alpha} =[\alpha_1, \alpha_2,...,\alpha_m]$. We simply use grid-search algorithm to find the best hyper-parameters in the classifier smoothing (C-OPT). For the input noise optimization (I-OPT), we use the Hill-Climbing algorithm to find the optimal hyper-parameters in the function for each input. During the algorithm execution, hyper-parameter for the input is iteratively updated if a better solution is found in each round, until convergence. The procedures for the Hill Climbing algorithm are summarized in Algorithm \ref{alg3} in Appendix \ref{apd: I-OPT alg}.

\vspace{0.05in}

\noindent\textbf{Optimality Validation of Noise PDF}. 
Finding an optimal noise PDF 
against a specific $\ell_p$ perturbation is important. Although Gaussian distribution can be used for defending against $\ell_2$ perturbations with tight certified radius, there is no evidence showing that Gaussian distribution is the optimal distribution against $\ell_2$ perturbations. Our UniCR can also somewhat validate the optimality of using different noise PDFs against different $\ell_p$ perturbations. 
For instance, Cohen et al. \cite{cohen2019certified}'s certified radius is tight for Gaussian noise against $\ell_2$ perturbations (see Figure \ref{fig:P-R curve}(b)). However, it is validated as not-optimal distribution against $\ell_2$ perturbations in our experiments (see Table \ref{tab: Classifier Opt}).







\section{Experiments}
\label{sec:exp}
\vspace{-2mm}

In this section, we thoroughly evaluate our UniCR framework, and benchmark with state-of-the-art certified defenses. First, we evaluate the \emph{universality} of UniCR by approximating the certified radii w.r.t. the probability $p_A$ using a variety of noise PDFs against $\ell_1$, $\ell_2$ and $\ell_\infty$ perturbations. 
Second, we validate the certified radii in existing works (results have been discussed and shown in Section \ref{sec:rbstudy}). Third, we evaluate our noise PDF optimization on three real-world datasets. 
Finally, we compare our best certified accuracy on CIFAR10 \cite{krizhevsky2009learning} and ImageNet \cite{ILSVRC15} with the state-of-the-art methods.

\vspace{-4mm}
\subsection{Experimental Setting}
\vspace{-2mm}
\noindent\textbf{Datasets}. We evaluate our performance on MNIST \cite{lecun2010mnist}, CIFAR10 \cite{krizhevsky2009learning} and ImageNet \cite{ILSVRC15}, which are common datasets to evaluate the certified robustness for image classification. 



\noindent\textbf{Metrics}. Following most existing works (e.g., Cohen et al. \cite{cohen2019certified}), we use the ``\textbf{approximate certified test set accuracy}'' at radius $R$ to evaluate the performance of certified robustness, which is defined as the fraction of the test set that the smoothed classifier will predict correctly against the perturbation within the radius $R$. The certified accuracy at different radii varies w.r.t. different noise variances, since the variance decides the trade-off between the radius and accuracy. A common way to compare the certified robustness of a noise on a dataset is to present the certified accuracy over a range of variances. However, there does not exist a unified metric used for measuring and comparing the certified accuracy. 
Therefore, to ensure a fair comparison, we also present the \textbf{robustness score} (Section 4.2) to measure the overall certified robustness across a range of noise variances.

\vspace{0.05in}

\noindent\textbf{Experimental Environment}. All the experiments were performed on the NSF Chameleon Cluster \cite{keahey2020lessons} with Intel(R) Xeon(R) Gold 6126 2.60GHz CPUs, 192G RAM, and NVIDIA Quadro RTX 6000 GPUs.


\vspace{-3mm}
\subsection{Universality Evaluation}
\label{exp: dif noises}
\vspace{-2mm}
As randomized smoothing derives certified robustness for any input and any classifier, our evaluation targets ``any noise PDF'' and ``any $\ell_p$ perubations''.

The certified radii of some noise PDFs, e.g., Gaussian noise against $\ell_2$ perturbations \cite{cohen2019certified}, Laplace noise against $\ell_1$ perturbations \cite{teng2020ell}, Pareto noise against $\ell_1$ perturbations \cite{yang2020randomized}, have been derived. These distributions have been verified by our UniCR framework in Figure \ref{fig:P-R curve}, where our certified radii highly approximate these theoretical radii. However, there are numerous noise PDFs of which the certified radii have not been theoretically studied, or they are difficult to derive. It is important to derive the certified radii of these distributions in order to find the optimal PDF against each of the $\ell_p$ perturbations. Therefore, we use our UniCR to approximately compute the certified radii of numerous distributions (including some mixture distributions, see Table \ref{tab:pdfs} in Appendix \ref{apd: pdflist}), some of which have not been studied before. Specifically, 
we evaluate different noise PDFs with the same variance, i.e.,  $\sigma=\mathbb{E}_{\epsilon \sim \mu}[\sqrt{\frac{1}{d} ||\epsilon||_2^2}]=1$. For those PDFs with multiple parameters, we set $\beta$ as $1.5$, $1.0$ and $0.5$ for General Normal, Pareto, and mixture distributions, respectively. Following Cohen et al. \cite{cohen2019certified}, and Yang et al. \cite{yang2020randomized}, we consider the binary case (Theorem \ref{thm2 binary}) and only compute the certified radius when $p_A \in (0.5,1.0]$.


\begin{figure}[!t]
    \begin{minipage}[t]{0.3\linewidth}
    \centering
    \includegraphics[width=38 mm]{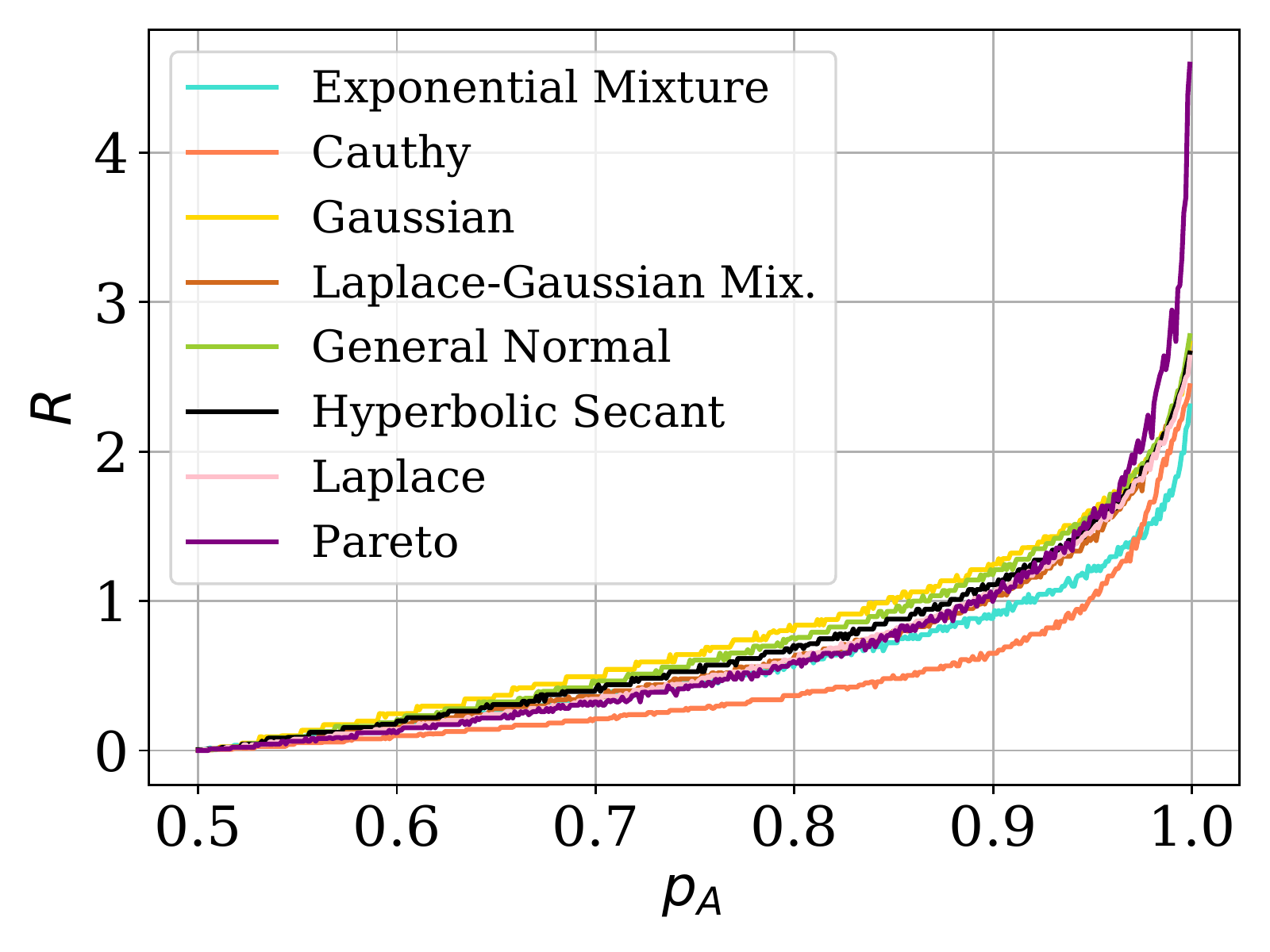}
    \vspace{-7mm}
    \caption{$R$-$p_A$ curve vs. $\ell_1$}
    \label{fig:p-r curve l1}
    \end{minipage}
    \hfill
    \begin{minipage}[t]{0.3\linewidth}
    \centering
    \includegraphics[width=38 mm]{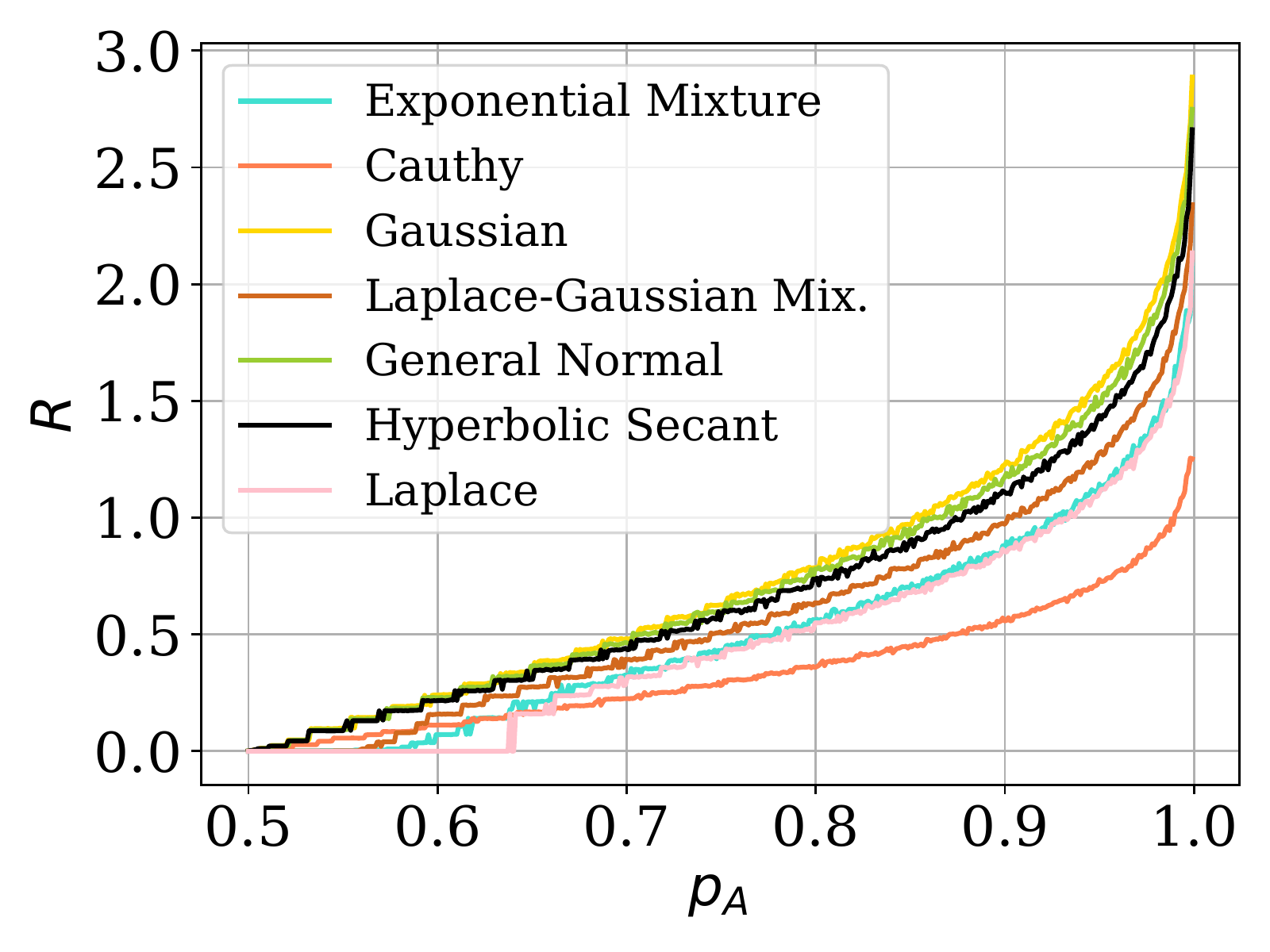}
    \vspace{-7mm}
    \caption{$R$-$p_A$ curve vs. $\ell_2$}
    \label{fig:p-r curve l2}
    \end{minipage}
    \hfill
    \begin{minipage}[t]{0.3\linewidth}
    \centering
    \includegraphics[width=38 mm]{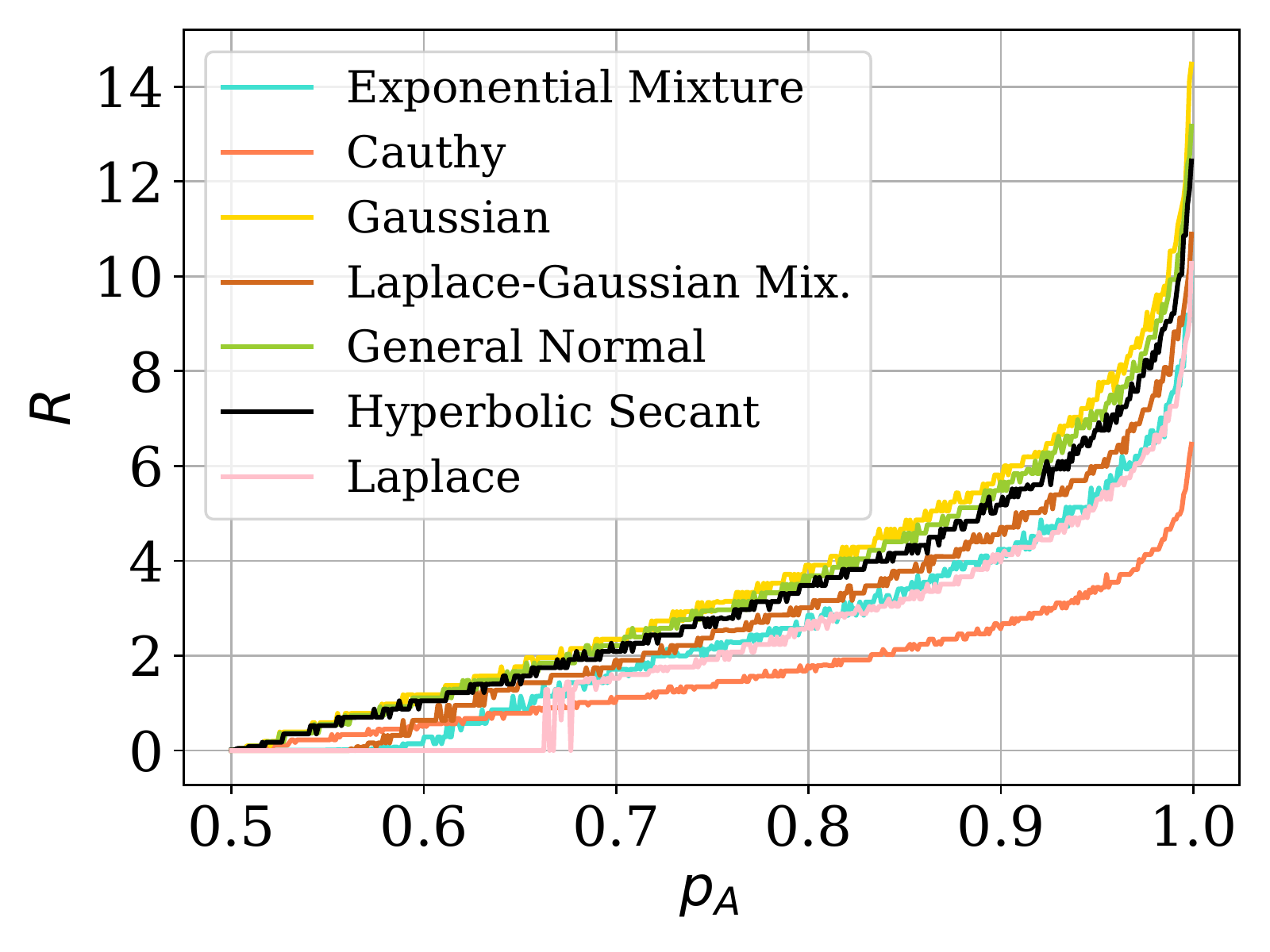}
    \vspace{-7mm}
    \caption{$R$-$p_A$ curve vs. $\ell_\infty$}
    \label{fig:p-r curve linf}
    \end{minipage}
    \vspace{-4mm}
\end{figure}





In Figure \ref{fig:p-r curve l1}-\ref{fig:p-r curve linf}, we plot the $R$-$p_A$ curves for the noise distributions listed in Table \ref{tab:pdfs} in Appendix \ref{apd: pdflist} against $\ell_1$, $\ell_2$ and $\ell_\infty$ perturbations. Specifically, we present the $\ell_\infty$  radius scaled by $\times 255$ to be consistent with the existing works \cite{zhang2020black}. We observe that for all $\ell_p$ perturbations, the Gaussian noise generates the largest certified radius for most of the $p_A$ values. All the noise distribution has very close $R$-$p_A$ curves except the Cauthy distribution. We also notice that when $p_A$ is low against $\ell_2$ and $\ell_\infty$ perturbations, our UniCR cannot find the certified radius for the Laplace-based distributions, e.g., Laplace distribution, and Gaussian-Laplace mixture distribution. This matches the findings on injecting Laplace noises for certified robustness in Yang et al. \cite{yang2020randomized}---The certified radii for Laplace noise against $\ell_2$ and $\ell_\infty$ perturbations are difficult to derive. 

We also conduct experiments to illustrate UniCR's universality in deriving $\ell_p$ norm certified radius for any real number $p>0$ in Appendix \ref{sec: dif p}. Besides, we also conduct fine-grained evaluations on General Normal, Laplace-Gaussian Mixture, and Exponential Mixture noises with various $\beta$ parameters (See Figure \ref{fig:P-R galary} in Appendix \ref{apd:fine-grain evaluation}), and we can draw similar observations from such results.



\begin{table*}[!t]
    \centering
    \caption{\textmd{Classifier-input noise optimization (C-OPT). We show the Robustness Score w.r.t. different $\beta$ settings of General Normal distribution ($\propto e^{-|x/\alpha|^{\beta}}$). The~$\sigma$ is set to $1.0$ for all distributions by adjusting the $\alpha$ parameter in General Normal.}}
    \vspace{-2mm}
    \resizebox{\linewidth}{!}{
    \begin{tabular}{c c c c c c c c c c c c c c c c c}
    \hline
         $\beta$ & 0.25 & 0.50 & 0.75 & 1.00 & 1.25& 1.50& 1.75& 2.00 & 2.25& 2.50& 2.75& 3& 4.00 & 5.00  \\
    \hline
        vs. $\ell_1$&1.8999 &2.6136 &\textbf{2.8354} &2.7448 &2.5461 &2.4254 &2.3434 &2.2615 &2.2211 &2.1730 &2.1081 &2.0679 & 1.9610 &1.8925   \\
        vs. $\ell_2$&0.0000 &0.0003 &1.0373 &1.5954 &1.9255 &2.0882 &2.1746 &2.1983 &\textbf{2.2081} &2.1771 &2.1184 &2.0655 &1.8857  &1.7296  \\
        vs. $\ell_\infty $&0.0000 &0.0109 &0.0420 &0.0641 &0.0771 &0.0839 &0.0871 &0.0879 &\textbf{0.0880} &0.0870 &0.0847 &0.0825 &0.0758  &0.0693  \\ 
    \hline
    \end{tabular}}
    \label{tab: Classifier Opt}
    \vspace{-4mm}
\end{table*}

\vspace{-3mm}
\subsection{Optimizing Certified Radius with C-OPT}
\vspace{-2mm}
\label{exp: global noise optimization}
We next show how C-OPT uses UniCR to  improve the certification against any $\ell_p$ perturbations. 
Recall that tight certified radii against $\ell_1$ and $\ell_2$ perturbations can be derived by the Laplace \cite{teng2020ell} and Gaussian \cite{cohen2019certified} noises, respectively. However, there does not exist any theoretical study showing that Laplace and Gaussian noises are the optimal noises against $\ell_1$ and $\ell_2$ perturbations, respectively. \cite{yang2020randomized,zhang2020black} have identified that there exists other better noise for $\ell_1$ and $\ell_2$ perturbations. Therefore, we use our C-OPT to explore the optimal distribution for each $\ell_p$ perturbation. Since the commonly used noise, e.g., Laplace and Gaussian noises, are only special cases of the General Normal Distribution ($\propto e^{-|x/\alpha|^\beta}$), we will find the optimal parameters $\alpha$ and $\beta$ that generate the best noises for maximizing certified radius against each $\ell_p$ perturbation. 

In the experiments, we use the grid search method to search the best parameters. We choose $\beta$ as the main parameter, and $\alpha$ will be set to satisfy $\sigma=1$. Specifically, we evaluate C-OPT on the MNIST dataset, where 
we train a model on the training set for each round of the grid search and certify $1,000$ images in the test set. Specifically, for each pair of parameters $\alpha$ and $\beta$ in the grid search, we train a Multiple Layer Perception on MNIST with the smoothing noise. Then, we compute the robustness score over a set of $\sigma=[0.12, 0.25, 0.50, 1.00]$. When approximating the certified radius with UniCR, we set the sampling number as $1,000$ in the Monte Carlo method. The results are shown in Table \ref{tab: Classifier Opt}. 

We observe that the best $\beta$ for $\ell_1$-norm is $0.75$ in the grid search. It indicates that the Laplace noise ($\beta=1$) is not the optimal noise against $\ell_1$ perturbations. A slightly smaller $\beta$ can provide a better trade-off between the certified radius and accuracy (measured by the robustness score). When $\beta<1.0$, the radius is observed to be larger than the radius derived with Laplace noise at $p_A \thickapprox 1$ (see Figure \ref{fig:P-R galary}(a)). Since $p_A$ on MNIST is always high, the noise distribution with $\beta=0.75$ will give a larger radius at most cases. Furthermore, we observe that the best performance against $\ell_2$ and $\ell_\infty$ are given by $\beta=2.25$, showing that the Gaussian noise is not the optimal noise against $\ell_2$ and $\ell_\infty$ perturbations, either. 

\vspace{-2mm}
\subsection{Optimizing Certified Radius with I-OPT}
\label{exp: local noise optimization}
\vspace{-2mm}
The optimal noises for different inputs are different. We customize the noise for each input using the I-OPT. Specifically, we adapt the hyper-parameters in the noise PDF to find the optimal noise distribution for each input (the classifier is smoothed by a standard method such as Cohen's \cite{cohen2019certified}). 

We perform I-OPT for noise PDF optimization with a Gaussian-trained ResNet50 classifier ($\sigma=1$) on ImageNet. We compare our derived radius with the theoretical radius in \cite{yang2020randomized,cohen2019certified}. We use the General Normal distribution to generate the noise for input certification since it provides a new parameter dimension for tuning. We tune the parameters $\alpha$ and $\beta$ in $e^{-|x/\alpha|^\beta}$. The Gaussian distribution is only a specific case of the General Normal distribution with $\beta=2$. In the two baselines \cite{yang2020randomized,cohen2019certified}, they set $\sigma=1$ and $\beta=2$, respectively. 
In the I-OPT, we initialize the noise with the same setting, but optimize the noise for each input. When approximating the certified radius with UniCR, we generate $1,000$ Monte Carlo samples for ImageNet.

\begin{table}[!t]
    \centering
    \scriptsize
      \caption{\textmd{Average Certified Radius with 
      Input Noise Optimization (I-OPT) 
      against $\ell_1$, $\ell_2$ and $\ell_\infty$ perturbations on ImageNet.}}
      \vspace{-2mm}
    \begin{tabular}{C{3 cm}C{1.5 cm}C{1.5 cm}C{1.5 cm}C{1.5 cm}C{1.5 cm}  }
    \hline
    Top $\ell_1$ radius  &$20\%$ &$40\%$ &$60\%$ &$80\%$ &$100\%$\\
    \hline
         Yang's Gaussian \cite{yang2020randomized}     &\textbf{2.44} &2.10 &1.59 &1.19 &0.95\\
         Ours with I-OPT   &2.36 &\textbf{2.11} &\textbf{1.64} &\textbf{1.23} &\textbf{0.98}  \\
    \hline
    Top $\ell_2$ radius &$20\%$ &$40\%$ &$60\%$ &$80\%$ &$100\%$ \\
    \hline
         Cohen's Gaussian \cite{cohen2019certified}    &\textbf{2.43} &2.10 &1.58 &1.19 &0.95 \\
         Ours with I-OPT   &2.36 &\textbf{2.11} &\textbf{1.64} &\textbf{1.23} &\textbf{0.98}\\
    \hline

    Top $\ell_\infty$ radius $\times 255$ &$20\%$ &$40\%$ &$60\%$ &$80\%$ &$100\%$ \\
    \hline
         Yang's Gaussian \cite{yang2020randomized} &1.60 &1.38 &1.04 & 0.78 & 0.63\\
         Ours with I-OPT &\textbf{1.75} &\textbf{1.54} &\textbf{1.20} & \textbf{0.90}& \textbf{0.72} \\
    \hline
    \end{tabular}
    \label{tab:Input Optimization}
    \vspace{-5mm}
\end{table}

Table \ref{tab:Input Optimization} presents the average values of the top $20\%$-$100\%$ certified radius (the higher the better). It shows that our method with I-OPT significantly improves the certified radius over the tight certified radius. This is because our I-OPT provides a personalized noise optimization to each input (see Figure \ref{fig:an example} in Appendix \ref{apd:examples} for the illustration). 


\vspace{-3mm}
\subsection{Best Performance Comparison}
\vspace{-2mm}
In this section, we compare our best performance with the state-of-the-art certified defense methods on the CIFAR10 and ImageNet datasets. Following the setting in \cite{cohen2019certified}, we use a ResNet110 \cite{he2016deep} classifier for the CIFAR10 dataset and a ResNet50 \cite{he2016deep} classifier for the ImageNet dataset. We evaluate the certification performance with the noise PDF of a range of variances $\sigma$. The $\sigma$ is set to vary in $[0.12, 0.25, 0.5, 1.0]$ for CIFAR10 and $[0.25, 0.5, 1.0]$ for ImageNet. We also present the Robustness Score based on this set of variances. We use the General Normal distribution and perform the I-OPT. The distribution is initialized with the same setting in the baselines, e.g., $\beta=1$ (or 2) for Laplace (Gaussian) baseline.
We benchmark it with the Laplace noise \cite{teng2020ell} on CIFAR10 when against $\ell_1$ perturbations; and the Gaussian noise \cite{cohen2019certified,yang2020randomized} on both CIFAR10 and ImageNet against all $\ell_p$ perturbations. For both our method and baselines, we use $1,000$ and $4,000$ Monte Carlo samples on ImageNet and CIFAR10, respectively, due to different scales, and the certified accuracy is computed over the certified radius of $500$ images randomly chosen in the test set for both CIFAR10 and ImageNet.

\begin{table}[!t]
    \centering\scriptsize
        \caption{\textmd{Certified accuracy and robustness score against $\ell_1$, $\ell_2$ and $\ell_\infty$ perturbations on CIFAR10. Ours: General Normal with I-OPT.}}
        \vspace{-2mm}
    \begin{tabular}{C{3 cm}C{1.3 cm}C{1.3 cm}C{1.3 cm}C{1.3 cm}C{1.3 cm}C{1.3cm}}
    \hline
    $\ell_1$  radius 																				&0.50   	&1.00  	&1.50 	&2.00  &2.50   & $R_{score}$   \\
    \hline
         Teng's Laplace \cite{teng2020ell}   							 	&39.2    	&17.2   	 &10.0        &6.0     &2.8     & 0.5606  \\
         Ours        								   &\textbf{45.8}     &\textbf{22.4}     &\textbf{14.8}      &\textbf{8.2}    &\textbf{3.6}    &\textbf{0.7027}  \\
    \hline
    $\ell_2$  radius 																			 	&0.50   	&1.00  	&1.50  	&2.00   &2.50   & $R_{score}$   \\
    \hline
         Cohen's Gaussian \cite{cohen2019certified}    			  	&38.6   	&17.4   	&8.6      &3.4      & 1.6  & 0.5392  \\
         Ours     								   &\textbf{48.4}   &\textbf{26.8}     &\textbf{16.6}    &\textbf{6.8}   &\textbf{2.0}    &\textbf{0.7141}  \\
    \hline
    $\ell_\infty$  radius 																		  &$\frac{2}{255}$   &$\frac{4}{255}$  & $\frac{6}{255}$ & $\frac{8}{255}$  & $\frac{10}{255}$    &$R_{score}$  \\
    \hline
         Yang's Gaussian \cite{yang2020randomized}  	   &43.6     &21.8   &10.8       &5.6    &2.6    &0.0098  \\
         Ours  							  &\textbf{53.4}     &\textbf{30.4}     &\textbf{21.2}   &\textbf{13.2}      &\textbf{5.6}    &\textbf{0.0136}  \\
    \hline
    \end{tabular}
    \label{tab:CIFAR10_comparison}
    \vspace{-2mm}
\end{table}

\begin{table}[!t]
    \centering
    \scriptsize
        \caption{\textmd{Certified accuracy and robustness score against $\ell_1$, $\ell_2$ and $\ell_\infty$ perturbations on ImageNet (Teng's Laplace \cite{teng2020ell} is not available). Ours: General Normal with I-OPT.}}
        \vspace{-2mm}
    \begin{tabular}{C{3 cm}C{1.3 cm}C{1.3 cm}C{1.3 cm}C{1.3 cm}C{1.3 cm}C{1.3cm}  }
    \hline
    $\ell_1$  radius &0.50 &1.00 &1.50 &2.00 &2.50 & $R_{score}$\\
    \hline
         Yang's Gaussian \cite{yang2020randomized}    &58.8 &45.6 &34.6 &27.0 &0.0  &1.0469  \\
         Ours&\textbf{63.4} &\textbf{49.6} &\textbf{36.8} &\textbf{29.6} &\textbf{6.6}     &\textbf{1.1385}  \\
    \hline
    $\ell_2$  radius &0.50 &1.00 &1.50 &2.00 &2.50 & $R_{score}$\\
    \hline
         Cohen's Gaussian \cite{cohen2019certified}     &58.8 &44.2 &34.0 &27.0 &0.0 &1.0463  \\
         Ours       &\textbf{62.6} &\textbf{49.0} &\textbf{36.6} &\textbf{28.6} &\textbf{2.0}  &\textbf{1.0939}  \\
    \hline

    $\ell_\infty$  radius & $\frac{0.25}{255}$ &$\frac{0.50}{255}$ &$\frac{0.75}{255}$ &$\frac{1.00}{255}$ &$\frac{1.25}{255}$  & $R_{score}$ \\
    \hline
         Yang's Gaussian \cite{yang2020randomized} &63.6 &52.4 &39.8 &34.2 &28.0  &0.0027  \\
         Ours &\textbf{69.2}  &\textbf{57.4}  &\textbf{47.2}  &\textbf{38.2}  &\textbf{33.0}   &\textbf{0.0031} \\
    \hline
    \end{tabular}
    \label{tab:ImageNet_comparison}
    \vspace{-5mm}
\end{table}

The results are shown in Table \ref{tab:CIFAR10_comparison} and \ref{tab:ImageNet_comparison}. Both on CIFAR10 and ImageNet, we observe a significant improvement on the certified accuracy and robustness score. Specifically, on CIFAR10, our robustness score outperforms the state-of-the-arts by $25.34\%$, $32.44\%$ and $38.78\%$ against $\ell_1$, $\ell_2$ and $\ell_\infty$ perturbations, respectively. On ImageNet, our robustness score outperforms the state-of-the-arts by $8.75\%$, $4.55\%$ and $14.81\%$ against $\ell_1$, $\ell_2$ and $\ell_\infty$ perturbations, respectively.

\section{Related Work}
\label{sec:related}
\vspace{-2mm}






\noindent \textbf{Certified Defenses.} 
They aim to derive the certified robustness of machine learning classifiers against adversarial perturbations. 
Existing certified defenses methods can be classified into leveraging Satisfiability Modulo Theories~\cite{scheibler2015towards,carlini2017provably,ehlers2017formal,katz2017reluplex}, 
mixed integer-linear programming~\cite{cheng2017maximum,fischetti2018deep,bunel2018unified}, linear programming~\cite{wong2018provable,wong2018scaling}, semidefinite programming~\cite{raghunathan2018certified,raghunathan2018semidefinite}, dual optimization~\cite{dvijotham2018training,dvijotham2018dual}, 
 global/local Lipschitz constant methods \cite{gouk2021regularisation,tsuzuku2018lipschitz,anil2019sorting,DBLP:conf/icml/CisseBGDU17,DBLP:conf/nips/HeinA17}, 
abstract interpretation~\cite{gehr2018ai2,mirman2018differentiable,singh2018fast}, and layer-wise certification~\cite{mirman2018differentiable,singh2018fast,DBLP:journals/corr/abs-1810-12715,DBLP:conf/icml/WengZCSHDBD18,DBLP:conf/nips/ZhangWCHD18}, etc. However, none of these methods is able to scale to large models (e.g., deep neural networks) or is limited to specific type of network architecture, e.g., ReLU based networks.

Randomized smoothing was recently proposed certified defenses~\cite{lecuyer2019certified,li2020second,cohen2019certified,jia2019certified,wang2020certifying} that is scalable to large models and applicable to arbitrary classifiers.  
Lecuyer et al. \cite{lecuyer2019certified} proposed the first randomized smoothing-based certified defense via differential privacy \cite{DworkR14}. 
Li et al. \cite{li2020second} proposed a stronger guarantee for Gaussian noise using information theory. The first tight robustness guarantee 
against $l_2$-norm perturbation for Gaussian noise was developed by Cohen et al. \cite{cohen2019certified}. After that, a series follow-up works have been proposed for other $\ell_p$-norms, e.g., $\ell_1$-norm \cite{teng2020ell}, $\ell_0$-norm \cite{DBLP:conf/aaai/LevineF20,DBLP:conf/nips/LeeYCJ19,jia2022almost}, etc. However, all these methods are limited to guarantee the robustness against only a specific $\ell_p$-norm perturbation.

\vspace{0.05in}

\noindent \textbf{Universal Certified Defenses.} More recently, several works \cite{zhang2020black,yang2020randomized} aim to provide more universal certified robustness schemes for all $\ell_p$-norms. Yang et al. \cite{yang2020randomized} proposed a level set method and a differential method to derive the upper bound and lower bound of the certified radius, while the derivation is relying on the case-by-case theoretical analysis. Zhang et al. \cite{zhang2020black} proposed a black-box optimization scheme that automatically computes the certified radius, but the solvable distribution is limited to $\ell_p$-norm. Dvijotham et al. \cite{DBLP:conf/iclr/DvijothamHBKQGX20} proposes a general certified defense based on the $f$-divergence, but fails to provide a tight certification. Croce et al. \cite{croce2019provable} derived the certified radius for any $\ell_p$-norm ($p\geq 1$),  
but the robustness guarantee can only be applied to ReLU classifiers. 
Our UniCR framework can automate the robustness certification for any classifier against any $l_p$-norm perturbation with any noise PDF. 

\vspace{0.05in}

\noindent \textbf{Certified Defenses with Optimized Noise PDFs/Distributions.} 
Yang et al.~\cite{yang2020randomized} proposed to use the Wulff Crystal theory \cite{wul1901frage} to find optimal noise distributions.  Zhang et al. \cite{zhang2020black} claimed that the optimal noise should have a more central-concentrated distribution from the optimization perspective. However, no existing works provide quantitative solutions to find optimal noise distributions. 
We propose the {\bf C-Opt} and  {\bf I-Opt} schemes to quantitatively optimize the noisy PDF in our UniCR framework and provide better certified robustness. 
Table \ref{table:related work} summarizes the differences in all the closely-related works. 
\section{Conclusion}
\label{sec:concl}
\vspace{-2mm}
Building effective certified defenses for neural networks against adversarial perturbations have attracted significant interests recently. However, the state-of-the-art methods lack universality to certify robustness. We propose the first randomized smoothing-based universal certified robustness approximation framework against any $\ell_p$ perturbations with any continuous noise PDF.
Extensive evaluations on multiple image datasets demonstrate the effectiveness of our UniCR framework and its advantages over the state-of-the-art certified defenses against any $\ell_p$ perturbations. 
%
%
\vspace{-4mm}
\section*{Acknowledgement}
\vspace{-4mm}
This work is partially supported by the National Science Foundation (NSF) under the Grants No. CNS-2046335 and CNS-2034870, as well as the Cisco Research Award. In addition, results presented in this paper were obtained using the Chameleon testbed supported by the NSF. Finally, the authors would like to thank the anonymous reviewers for their constructive comments.

\bibliographystyle{splncs04}
\bibliography{UCR}

\begin{thebibliography}{10}
\providecommand{\url}[1]{\texttt{#1}}
\providecommand{\urlprefix}{URL }
\providecommand{\doi}[1]{https://doi.org/#1}

\bibitem{andriushchenko2020square}
Andriushchenko, M., Croce, F., Flammarion, N., Hein, M.: Square attack: a
  query-efficient black-box adversarial attack via random search. In: European
  Conference on Computer Vision. pp. 484--501. Springer (2020)

\bibitem{anil2019sorting}
Anil, C., Lucas, J., Grosse, R.: Sorting out lipschitz function approximation.
  In: International Conference on Machine Learning. pp. 291--301. PMLR (2019)

\bibitem{bunel2018unified}
Bunel, R.R., Turkaslan, I., Torr, P.H., Kohli, P., Mudigonda, P.K.: A unified
  view of piecewise linear neural network verification. In: NeurIPS (2018)

\bibitem{carlini2017provably}
Carlini, N., Katz, G., Barrett, C., Dill, D.L.: Provably minimally-distorted
  adversarial examples. arXiv preprint arXiv:1709.10207  (2017)

\bibitem{carlini2017towards}
Carlini, N., Wagner, D.: Towards evaluating the robustness of neural networks.
  In: 2017 ieee symposium on security and privacy (sp). pp. 39--57. IEEE (2017)

\bibitem{chen2020hopskipjumpattack}
Chen, J., Jordan, M.I., Wainwright, M.J.: Hopskipjumpattack: A query-efficient
  decision-based attack. In: 2020 IEEE Symposium on Security and Privacy (2020)

\bibitem{chen2017zoo}
Chen, P.Y., Zhang, H., Sharma, Y., Yi, J., Hsieh, C.J.: Zoo: Zeroth order
  optimization based black-box attacks to deep neural networks without training
  substitute models. In: 10th ACM workshop on artificial intelligence and
  security (2017)

\bibitem{cheng2017maximum}
Cheng, C.H., N{\"u}hrenberg, G., Ruess, H.: Maximum resilience of artificial
  neural networks. In: International Symposium on Automated Technology for
  Verification and Analysis. pp. 251--268. Springer (2017)

\bibitem{DBLP:conf/icml/CisseBGDU17}
Ciss{\'{e}}, M., Bojanowski, P., Grave, E., Dauphin, Y.N., Usunier, N.:
  Parseval networks: Improving robustness to adversarial examples. In:
  Proceedings of the 34th International Conference on Machine Learning (2017)

\bibitem{co2019procedural}
Co, K.T., Mu{\~n}oz-Gonz{\'a}lez, L., de~Maupeou, S., Lupu, E.C.: Procedural
  noise adversarial examples for black-box attacks on deep convolutional
  networks. In: ACM SIGSAC conference on computer and communications security
  (2019)

\bibitem{cohen2019certified}
Cohen, J., Rosenfeld, E., Kolter, Z.: Certified adversarial robustness via
  randomized smoothing. In: International Conference on Machine Learning (2019)

\bibitem{croce2019provable}
Croce, F., Hein, M.: Provable robustness against all adversarial
  {\textdollar}l{\_}p{\textdollar}-perturbations for
  {\textdollar}p{\textbackslash}geq 1{\textdollar}. In: {ICLR}. OpenReview.net
  (2020)

\bibitem{croce2020reliable}
Croce, F., Hein, M.: Reliable evaluation of adversarial robustness with an
  ensemble of diverse parameter-free attacks. In: International conference on
  machine learning. pp. 2206--2216. PMLR (2020)

\bibitem{dvijotham2018training}
Dvijotham, K., Gowal, S., Stanforth, R., et~al.: Training verified learners
  with learned verifiers. arXiv  (2018)

\bibitem{dvijotham2018dual}
Dvijotham, K., Stanforth, R., Gowal, S., Mann, T.A., Kohli, P.: A dual approach
  to scalable verification of deep networks. In: UAI (2018)

\bibitem{DBLP:conf/iclr/DvijothamHBKQGX20}
Dvijotham, K.D., Hayes, J., Balle, B., Kolter, J.Z., Qin, C., Gy{\"{o}}rgy, A.,
  Xiao, K., Gowal, S., Kohli, P.: A framework for robustness certification of
  smoothed classifiers using f-divergences. In: ICLR (2020)

\bibitem{dvoretzky1956asymptotic}
Dvoretzky, A., Kiefer, J., Wolfowitz, J.: Asymptotic minimax character of the
  sample distribution function and of the classical multinomial estimator. The
  Annals of Mathematical Statistics pp. 642--669 (1956)

\bibitem{DworkR14}
Dwork, C., Roth, A.: The algorithmic foundations of differential privacy.
  Found. Trends Theor. Comput. Sci.  \textbf{9}(3-4),  211--407 (2014)

\bibitem{ehlers2017formal}
Ehlers, R.: Formal verification of piece-wise linear feed-forward neural
  networks. In: International Symposium on Automated Technology for
  Verification and Analysis. pp. 269--286. Springer (2017)

\bibitem{fischetti2018deep}
Fischetti, M., Jo, J.: Deep neural networks and mixed integer linear
  optimization. Constraints  \textbf{23}(3),  296--309 (2018)

\bibitem{gehr2018ai2}
Gehr, T., Mirman, M., Drachsler-Cohen, D., Tsankov, P., Chaudhuri, S., Vechev,
  M.: Ai2: Safety and robustness certification of neural networks with abstract
  interpretation. In: IEEE S \& P (2018)

\bibitem{gouk2021regularisation}
Gouk, H., Frank, E., Pfahringer, B., Cree, M.J.: Regularisation of neural
  networks by enforcing lipschitz continuity. Machine Learning
  \textbf{110}(2),  393--416 (2021)

\bibitem{DBLP:journals/corr/abs-1810-12715}
Gowal, S., Dvijotham, K., Stanforth, R., Bunel, R., Qin, C., Uesato, J.,
  Arandjelovic, R., Mann, T.A., Kohli, P.: On the effectiveness of interval
  bound propagation for training verifiably robust models. CoRR
  \textbf{abs/1810.12715} (2018), \url{http://arxiv.org/abs/1810.12715}

\bibitem{he2016deep}
He, K., Zhang, X., Ren, S., Sun, J.: Deep residual learning for image
  recognition. In: Proceedings of the IEEE conference on computer vision and
  pattern recognition. pp. 770--778 (2016)

\bibitem{DBLP:conf/nips/HeinA17}
Hein, M., Andriushchenko, M.: Formal guarantees on the robustness of a
  classifier against adversarial manipulation. In: Advances in Neural
  Information Processing Systems 30: Annual Conference on Neural Information
  Processing Systems 2017, December 4-9, 2017, Long Beach, CA, {USA}. pp.
  2266--2276 (2017)

\bibitem{hong2022eye}
Hong, H., Hong, Y., Kong, Y.: An eye for an eye: Defending against
  gradient-based attacks with gradients. arXiv preprint arXiv:2202.01117
  (2022)

\bibitem{jia2019certified}
Jia, J., Cao, X., Wang, B., Gong, N.Z.: Certified robustness for top-k
  predictions against adversarial perturbations via randomized smoothing. In:
  International Conference on Learning Representations (2019)

\bibitem{jia2020certified}
Jia, J., Wang, B., Cao, X., Gong, N.Z.: Certified robustness of community
  detection against adversarial structural perturbation via randomized
  smoothing. In: Proceedings of The Web Conference 2020. pp. 2718--2724 (2020)

\bibitem{jia2022almost}
Jia, J., Wang, B., Cao, X., Liu, H., Gong, N.Z.: Almost tight l0-norm certified
  robustness of top-k predictions against adversarial perturbations. In: ICLR
  (2022)

\bibitem{DBLP:conf/emnlp/JiaRGL19}
Jia, R., Raghunathan, A., G{\"{o}}ksel, K., Liang, P.: Certified robustness to
  adversarial word substitutions. In: {EMNLP/IJCNLP} (2019)

\bibitem{katz2017reluplex}
Katz, G., Barrett, C., Dill, D.L., Julian, K., Kochenderfer, M.J.: Reluplex: An
  efficient smt solver for verifying deep neural networks. In: International
  Conference on Computer Aided Verification. pp. 97--117. Springer (2017)

\bibitem{keahey2020lessons}
Keahey, K., Anderson, J., Zhen, Z., Riteau, P., Ruth, P., Stanzione, D., Cevik,
  M., Colleran, J., Gunawi, H.S., Hammock, C., Mambretti, J., Barnes, A.,
  Halbach, F., Rocha, A., Stubbs, J.: Lessons learned from the chameleon
  testbed. In: Proceedings of the 2020 USENIX Annual Technical Conference
  (USENIX ATC '20). USENIX Association (July 2020)

\bibitem{kennedy1995particle}
Kennedy, J., Eberhart, R.: Particle swarm optimization. In: Proceedings of
  ICNN'95-international conference on neural networks. IEEE (1995)

\bibitem{krizhevsky2009learning}
Krizhevsky, A., Hinton, G., et~al.: Learning multiple layers of features from
  tiny images  (2009)

\bibitem{lecun2010mnist}
LeCun, Y., Cortes, C., Burges, C.: Mnist handwritten digit database. ATT Labs
  [Online]. Available: http://yann.lecun.com/exdb/mnist  \textbf{2} (2010)

\bibitem{lecuyer2019certified}
Lecuyer, M., Atlidakis, V., Geambasu, R., Hsu, D., Jana, S.: Certified
  robustness to adversarial examples with differential privacy. In: 2019 IEEE
  Symposium on Security and Privacy (SP). pp. 656--672. IEEE (2019)

\bibitem{DBLP:conf/nips/LeeYCJ19}
Lee, G., Yuan, Y., Chang, S., Jaakkola, T.S.: Tight certificates of adversarial
  robustness for randomly smoothed classifiers. In: NeurIPS. pp. 4911--4922
  (2019)

\bibitem{DBLP:conf/aaai/LevineF20}
Levine, A., Feizi, S.: Robustness certificates for sparse adversarial attacks
  by randomized ablation. In: {AAAI}. pp. 4585--4593. {AAAI} Press (2020)

\bibitem{li2020second}
Li, B., Chen, C., Wang, W., Carin, L.: Second-order adversarial attack and
  certifiable robustness. arXiv preprint arXiv:2006.00731  (2020)

\bibitem{DBLP:conf/ndss/LiNPSKRS19}
Li, S., Neupane, A., Paul, S., Song, C., Krishnamurthy, S.V., Roy{-}Chowdhury,
  A.K., Swami, A.: Stealthy adversarial perturbations against real-time video
  classification systems. In: NDSS. The Internet Society (2019)

\bibitem{madry2018towards}
Madry, A., Makelov, A., Schmidt, L., Tsipras, D., Vladu, A.: Towards deep
  learning models resistant to adversarial attacks. In: International
  Conference on Learning Representations (2018)

\bibitem{mirman2018differentiable}
Mirman, M., Gehr, T., Vechev, M.: Differentiable abstract interpretation for
  provably robust neural networks. In: International Conference on Machine
  Learning (2018)

\bibitem{mohammady2020r2dp}
Mohammady, M., Xie, S., Hong, Y., Zhang, M., Wang, L., Pourzandi, M., Debbabi,
  M.: R2dp: A universal and automated approach to optimizing the randomization
  mechanisms of differential privacy for utility metrics with no known optimal
  distributions. In: Proceedings of the 2020 ACM SIGSAC Conference on Computer
  and Communications Security. pp. 677--696 (2020)

\bibitem{raghunathan2018certified}
Raghunathan, A., Steinhardt, J., Liang, P.: Certified defenses against
  adversarial examples. arXiv preprint arXiv:1801.09344  (2018)

\bibitem{raghunathan2018semidefinite}
Raghunathan, A., Steinhardt, J., Liang, P.S.: Semidefinite relaxations for
  certifying robustness to adversarial examples. In: NeurIPS (2018)

\bibitem{ILSVRC15}
Russakovsky, O., Deng, J., Su, H., Krause, J., Satheesh, S., Ma, S., Huang, Z.,
  Karpathy, A., Khosla, A., Bernstein, M., Berg, A.C., Fei-Fei, L.: {ImageNet
  Large Scale Visual Recognition Challenge}. International Journal of Computer
  Vision (IJCV)  \textbf{115}(3),  211--252 (2015).
  \doi{10.1007/s11263-015-0816-y}

\bibitem{scheibler2015towards}
Scheibler, K., Winterer, L., Wimmer, R., Becker, B.: Towards verification of
  artificial neural networks. In: MBMV. pp. 30--40 (2015)

\bibitem{singh2018fast}
Singh, G., Gehr, T., Mirman, M., P{\"u}schel, M., Vechev, M.: Fast and
  effective robustness certification. In: Proceedings of the 32nd International
  Conference on Neural Information Processing Systems. pp. 10825--10836 (2018)

\bibitem{teng2020ell}
Teng, J., Lee, G.H., Yuan, Y.: {\$}{\textbackslash}ell{\_}1{\$} adversarial
  robustness certificates: a randomized smoothing approach (2020)

\bibitem{tsuzuku2018lipschitz}
Tsuzuku, Y., Sato, I., Sugiyama, M.: Lipschitz-margin training: Scalable
  certification of perturbation invariance for deep neural networks. In:
  NeurIPS (2018)

\bibitem{wang2020certifying}
Wang, B., Cao, X., Gong, N.Z., et~al.: On certifying robustness against
  backdoor attacks via randomized smoothing. In: CVPR 2020 Workshop on
  Adversarial Machine Learning in Computer Vision (2020)

\bibitem{DBLP:conf/kdd/WangJCG21}
Wang, B., Jia, J., Cao, X., Gong, N.Z.: Certified robustness of graph neural
  networks against adversarial structural perturbation. In: {KDD}. {ACM} (2021)

\bibitem{DBLP:conf/icml/WengZCSHDBD18}
Weng, T., Zhang, H., Chen, H., Song, Z., Hsieh, C., Daniel, L., Boning, D.S.,
  Dhillon, I.S.: Towards fast computation of certified robustness for relu
  networks. In: Dy, J.G., Krause, A. (eds.) International Conference on Machine
  Learning (2018)

\bibitem{wong2018provable}
Wong, E., Kolter, J.Z.: Provable defenses against adversarial examples via the
  convex outer adversarial polytope. In: ICML (2018)

\bibitem{wong2019wasserstein}
Wong, E., Schmidt, F., Kolter, Z.: Wasserstein adversarial examples via
  projected sinkhorn iterations. In: International Conference on Machine
  Learning (2019)

\bibitem{wong2018scaling}
Wong, E., Schmidt, F.R., Metzen, J.H., Kolter, J.Z.: Scaling provable
  adversarial defenses. arXiv preprint arXiv:1805.12514  (2018)

\bibitem{wul1901frage}
Wul, G.: Zur frage der geschwindigkeit des wachstums und der auflosung der
  kristall achen. Z. Kristallogr  \textbf{34},  449--530 (1901)

\bibitem{xie2022universal}
Xie, S., Wang, H., Kong, Y., Hong, Y.: Universal 3-dimensional perturbations
  for black-box attacks on video recognition systems. In: In Proceedings of the
  43rd IEEE Symposium on Security and Privacy (Oakland'22) (2022)

\bibitem{yang2020randomized}
Yang, G., Duan, T., Hu, J.E., Salman, H., Razenshteyn, I., Li, J.: Randomized
  smoothing of all shapes and sizes. In: International Conference on Machine
  Learning. pp. 10693--10705. PMLR (2020)

\bibitem{zhang2020black}
Zhang, D., Ye, M., Gong, C., Zhu, Z., Liu, Q.: Black-box certification with
  randomized smoothing: A functional optimization based framework  (2020)

\bibitem{DBLP:conf/nips/ZhangWCHD18}
Zhang, H., Weng, T., Chen, P., Hsieh, C., Daniel, L.: Efficient neural network
  robustness certification with general activation functions. In: Neural
  Information Processing Systems (2018)

\end{thebibliography}

\newpage
\appendix
\label{apd:thm2}

\section{Preliminary}
\label{sec:preli}



We first briefly review the recent certified robustness scheme \cite{cohen2019certified} for a general classification problem by classifying data point in $\mathbb{R}^d$ to classes in $\mathcal{Y}$. Given an arbitrary base classifier $f$, it can be converted to a ``smoothed'' classifier \cite{cohen2019certified} $g$ by adding isotropic Gaussian noise to the input $x$:

\begin{equation}
    g(x)=\arg \max_{c\in \mathcal{Y}} \mathbb{P}(f(x+\epsilon)=c), where~~ \epsilon \sim \mathcal{N}(0,\sigma^2 I)
\label{equation1}
\end{equation}

\begin{Lem}{(\textbf{Neyman-Pearson Lemma})}
Let $X$ and $Y$ be random variables in $\mathbb{R}^d$ with densities $\mu_X$ and $\mu_Y$. Let $f:\mathbb{R}^d \rightarrow \{0,1\}$ be a random or deterministic function. Then:

(1) If $S=\{z\in \mathbb{R}^d: \frac{\mu_Y(z)}{\mu_X(z)}\leq{t}\}$ for some $t>0$ and $\mathbb{P}(f(X)=1)\ge \mathbb{P}(X\in S)$, then $\mathbb{P}(f(Y)=1) \ge \mathbb{P}(Y\in S)$;

(2) If $S=\{z\in \mathbb{R}^d: \frac{\mu_Y(z)}{\mu_X(z)}\ge {t}\}$ for some $t>0$ and $\mathbb{P}(f(X)=1)\leq \mathbb{P}(X\in S)$, then $\mathbb{P}(f(Y)=1) \leq \mathbb{P}(Y\in S)$.
\label{NP lemma}
\end{Lem}

With Lemma \ref{NP lemma}, Cohen \cite{cohen2019certified} derives the certified radius when the classifier is smoothed with the Gaussian noise. As shown in Theorem \ref{cohen thm}, when the smoothed classifier's prediction probabilities satisfy Equation (\ref{thm1: eq2}), the prediction result is guaranteed to be the most probable class $c_A$ when the perturbation is limited within a radius $R$ in $\ell_2$-norm.

\begin{thm}{(\textbf{Randomized Smoothing with Gaussian Noise  \cite{cohen2019certified}})}
Let $f:\mathbb{R}^d \rightarrow \mathcal{Y}$ be any deterministic or random function, and let $\epsilon \sim \mathcal{N}(0, \sigma^2 I)$. Denote $g$ as the smoothed classifier in Equation (\ref{equation1}), and the most probable and the second probable classes as $c_A, c_B \in \mathcal{Y}$, respectively. If the lower bound of the class $c_A$'s prediction probability $\underline{p_A} \in [0, 1]$, and the upper bound of the class $c_B$'s prediction probability $\overline{p_B}\in [0, 1] $ satisfy:
\begin{equation}
    \mathbb{P}(f(x+\epsilon)=c_A) \ge \underline{p_A} \ge \overline{p_B} \ge \max_{c \neq c_A} \mathbb{P}(f(x+\epsilon)=c)
\label{thm1: eq2}
\end{equation}
Then $g(x+\delta)=c_A$ for all $||\delta||_2 \leq R$, where
\begin{equation}
    R=\frac{\sigma}{2}(\Phi^{-1}(\underline{p_A})-\Phi^{-1}(\overline{p_B}))
\end{equation}
where $\Phi^{-1}$ is the inverse of the standard Gaussian CDF.
\label{cohen thm}
\end{thm}

\begin{proof}
See detailed proof in \cite{cohen2019certified}. 
\end{proof}



\section{Proofs}
\subsection{Proof of Theorem \ref{thm2}}
\label{apd: thm2 proof}

\begin{proof}

We prove the theorem based on Neyman-Pearson Lemma (Lemma \ref{NP lemma}).






\vspace{0.05in}

Let $x:=x_0+\epsilon$ be the random variable that follows any continuous distribution. $\delta$ be the perturbation added to the input image. $y=x_0+\epsilon+\delta$ is the perturbed random variable. Thus, $x$ and $y$ are random variables with densities $\mu_x$ and $\mu_y$. Define sets:
\begin{equation}
    A:=\{z:\frac{\mu_y(z)}{\mu_x(z)}\leq t_A\}
\label{eq ap 1}
\end{equation}
\begin{equation}
    B:=\{z:\frac{\mu_y(z)}{\mu_x(z)}\ge t_B\}
\label{eq ap 2}
\end{equation} where $t_A$ and $t_B$ are picked to suffice:

\begin{equation}
    \mathbb{P}(x\in A)=\underline{p_A}
\label{eq ap 3}
\end{equation}
\begin{equation}
    \mathbb{P}(x\in B)=\overline{p_B}
\label{eq ap 4}
\end{equation}

Suppose $c_A\in \mathcal{Y}$ and $\underline{p_A},\overline{p_B} \in [0,1]$ satisfy:

\begin{equation}
\mathbb{P}(f(x+\epsilon)=c_A)\ge \underline{p_A}\ge \overline{p_B}\ge \max_{c\neq c_A}\mathbb{P}(f(x+\epsilon)=c)
\label{eq ap 5}
\end{equation}

Since $\mathbb{P}(f(x+\epsilon)=c_A) \ge \underline{p_A}=\mathbb{P}(x\in A)$ and $A=\{z:\frac{\mu_Y(z)}{\mu_X(z)}\leq t_A\}$, using Neyman-Pearson Lemma (Lemma \ref{NP lemma}), we have:
\begin{equation}
    \mathbb{P}(f(y)=c_A)\ge \mathbb{P}(y\in A)
\label{eq ap 6}
\end{equation}

Similarly, we have:
\begin{equation}
    \mathbb{P}(f(y)=c_B)\leq \mathbb{P}(y\in B)
\label{eq ap 7}
\end{equation}

To guarantee $\mathbb{P}(f(y)=c_A)\ge \mathbb{P}(f(y)=c_B)$, we need
\begin{equation}
    \mathbb{P}(f(y)=c_A)\ge \mathbb{P}(y\in A) \ge \mathbb{P}(y\in B) \ge \mathbb{P}(f(y)=c_B)
\label{eq ap 8}
\end{equation}

In summary, to guarantee the certified robustness on class $A$, Equation (\ref{eq ap 1}), (\ref{eq ap 2}), (\ref{eq ap 3}), (\ref{eq ap 4}), (\ref{eq ap 8}) must be satisfied. The conditions can be rewritten as:

\begin{equation}
    \mathbb{P}(\frac{\mu_y(x)}{\mu_x(x)}\leq t_A)=\underline{p_A}
\label{eq ap 9}
\end{equation}
\begin{equation}
        \mathbb{P}(\frac{\mu_y(x)}{\mu_x(x)}\ge t_B)=\overline{p_B}
\label{eq ap 10}
\end{equation}
\begin{equation}
    \mathbb{P}(\frac{\mu_y(y)}{\mu_x(y)}\leq t_A)\ge \mathbb{P}(\frac{\mu_y(y)}{\mu_x(y)}\ge t_B)
\label{eq ap 11}
\end{equation} 

\vspace{0.05in}
where Equation (\ref{eq ap 9}) is from Equation (\ref{eq ap 1}) and Equation (\ref{eq ap 3}), Equation (\ref{eq ap 10}) is from Equation (\ref{eq ap 2}) and Equation (\ref{eq ap 4}), and Equation (\ref{eq ap 11}) is from Equation (\ref{eq ap 8}).

Considering the relationship $y=x+\delta$, we can derive:
\begin{equation}
    \mu_y(x)=\mu_x(x-\delta)
\label{eq ap 12}
\end{equation}
\begin{equation}
    \mu_x(y)=\mu_x(x+\delta)
\label{eq ap 13}
\end{equation}
\begin{equation}
    \mu_y(y)=\mu_x(y-\delta)=\mu_x(x)
\label{eq ap 14}
\end{equation}

Thus, the conditions (11), (12) and (13) can be rewritten as:
\begin{equation}
    \mathbb{P}(\frac{\mu_x(x-\delta)}{\mu_x(x)}\leq t_A)=\underline{p_A}
\label{eq ap 15}
\end{equation}
\begin{equation}
        \mathbb{P}(\frac{\mu_x(x-\delta)}{\mu_x(x)}\ge t_B)=\overline{p_B}
\label{eq ap 16}
\end{equation}
\begin{equation}
    \mathbb{P}(\frac{\mu_x(x)}{\mu_x(x+\delta)}\leq t_A)\ge \mathbb{P}(\frac{\mu_x(x)}{\mu_x(x+\delta)}\ge t_B)
\label{eq ap 17}
\end{equation}

Any perturbation $\delta$ satisfying these conditions will not fool the smoothed classifier. In this case, these conditions construct a robustness area in $\delta$ space. If we want to find a $\ell_p$ ball within which the prediction is constant, the $l_p$ ball should be in this robustness area. Therefore, the certified radii is the minimum $||\delta||_p$ on the boundary of this robustness area. In this case, the $\ell_p$ ball is exactly the maximum inscribed ball in the robustness area. Also, $x$ can be replace by $\epsilon$ in these conditions since it is in the fraction, which means the optimization is independent to the input if given $\underline{p_A}$ and $\overline{p_B}$. Therefore, the whole optimization problem is summarized as:
\begin{mini*}|l|
  {\delta}{R=||\delta||_p}{}{}
  \addConstraint{    \mathbb{P}(\frac{\mu_x(x-\delta)}{\mu_x(x)}\leq t_A)=\underline{p_A}}
  \addConstraint{    \mathbb{P}(\frac{\mu_x(x-\delta)}{\mu_x(x)}\ge t_B)=\overline{p_B}}
  \addConstraint{    \mathbb{P}(\frac{\mu_x(x)}{\mu_x(x+\delta)}\leq t_A)= \mathbb{P}(\frac{\mu_x(x)}{\mu_x(x+\delta)}\ge t_B)}
\end{mini*}

If the noise is isotropic, each dimension is independent,
\begin{equation}
    \mu_x(x)=\prod_{i=1}^{d}\mu_x(x_j)
\label{eq ap 18}
\end{equation}

Thus, conditions for the isotropic noise can be rewritten as:

\begin{equation}
    \mathbb{P}(\prod_{j=1}^{d}\frac{\mu_x(x_j-\delta_j)}{\mu_x(x_j)}\leq t_A)=\underline{p_A}
\label{eq ap 19}
\end{equation}
\begin{equation}
    \mathbb{P}(\prod_{j=1}^{d}\frac{\mu_x(x_j-\delta_j)}{\mu_x(x_j)}\ge t_B)=\overline{p_B}
\label{eq ap 20}
\end{equation}
\begin{equation}
    \mathbb{P}(\prod_{j=1}^{d}\frac{\mu_x(x_j)}{\mu_x(x_j+\delta_j)}\leq t_A)= \mathbb{P}(\prod_{j=1}^{d}\frac{\mu_x(x_j)}{\mu_x(x_j+\delta_j)}\ge t_B)
\label{eq ap 21}
\end{equation}

Thus, this completes the proof. 
\end{proof}

\subsection{Binary Case for Theorem 2}
\label{apd thm2 binary}

\begin{thm}{(\textbf{Universal Certified Robustness (Binary Case)})}
Let $f:\mathbb{R}^d \rightarrow \mathcal{Y}$ be any deterministic or random function, and let $\epsilon$ follows any continuous distribution. Let $g$ be defined as in (\ref{general randomized smoothing}). Suppose the most probable class $c_A \in \mathcal{Y}$ and the lower bound of the probability $\underline{p_A}$ satisfy:
\begin{equation}
    \mathbb{P}(f+\epsilon)=c_A \ge \underline{p_A} \ge \frac{1}{2}
\end{equation}

Then $g(x+\delta)=c_A$ for all $||\delta||_p \leq R$, where R is given by the optimization:

\begin{mini*}|l|
  {\delta}{R=||\delta||_p}{}{}
  \addConstraint{    \mathbb{P}(\frac{\mu_x(x-\delta)}{\mu_x(x)}\leq t_A)=\underline{p_A}}
  \addConstraint{    \mathbb{P}(\frac{\mu_x(x)}{\mu_x(x+\delta)}\leq t_A)=\frac{1}{2}}
\end{mini*}

\label{thm2 binary}
\end{thm}

\subsection{UniCR (Binary Case)}

Similar to the binary case of Theorem \ref{thm2}, the binary case of the two-phase optimization can be easily derived:

\begin{align}
    &R=||\lambda \bm{\delta}||_p, where \ \bm{\delta} \in \argmin_{\delta} ||\lambda \bm{\delta} ||_p  \nonumber\\
    &s.t.~~ \ \lambda = \argmin_{\lambda} |K| \nonumber\\
    & \qquad~{\mathbb{P}(\frac{\mu_x(x-\lambda\bm{\delta})}{\mu_x(x)}\leq t_A)=\underline{p_A}} \nonumber\\
    &\qquad ~ K=\mathbb{P}(\frac{\mu_x(x)}{\mu_x(x+\lambda\bm{\delta})}\leq t_A)-\frac{1}{2} \nonumber\\
    &\qquad ~\underline{p_A} \ge \frac{1}{2} \nonumber
\label{opt 2 binary}
\end{align}

\subsection{UniCR Bound}
\label{apd: confidence}
The certified radius $R$ approximated by the two-phase optimization is tight if achieving the optimality. Under this assumption, we analysis the confidence bound for the certification. We follow \cite{cohen2019certified} to compute the probabilities $\underline{p_A}$ and $\overline{p_B}$ using Monte Carlo method with sample number n. The confidence is $1-\alpha_0$, where $p_A>=\alpha_0^{1/n}$. To estimate the auxiliary parameters $t_A$ and $t_B$, we use Dvoretzky–Kiefer–Wolfowitz inequality \cite{dvoretzky1956asymptotic} to bound the CDFs of the random variables $\frac{\mu_x(x-\lambda\delta)}{\mu_x(x)}$ and $\frac{\mu_x(x)}{\mu_x(x+\lambda\delta)}$, then determine the $t_A$ and $t_B$ using Algorithm \ref{alg1}.

\begin{Lem}{(\textbf{Dvoretzky–Kiefer–Wolfowitz inequality}(restate))}
Let $X_1, X_2, ..., X_n$ be real-valued independent and identically distributed random variables with cumulative distribution function $F(\cdot)$, where $n \in \mathbb{N}$.Let $F_n$ denotes the associated empirical distribution function defined by
\begin{equation}
    F_n(x)=\frac{1}{n} \sum_{i=1}^{n} \mathbf{1}_{\{X_i<=x\}}, x \in \mathbb{R}
\end{equation}

The Dvoretzky–Kiefer–Wolfowitz inequality bounds the probability that the random function $F_n$ differs from $F$ by more than a given constant $\Delta \in \mathbb{R}^+$ :

\begin{equation}
    \mathbb{P}(\sup_{x\in\mathbb{R}}|F_n(x)-F(x)|>\Delta) \leq 2e^{-2n\Delta^2}
\end{equation}
\label{DKW lemma restate}
\end{Lem}

We use the Lemma \ref{DKW lemma restate} to estimate the CDFs in algorithm \ref{alg1}. In condition \ref{opt 2}, we need to estimate $4$ probabilities with confidence $1-2e^{-2n\Delta^2}$ as well as the $p_A$ and $p_B$ with confidence $1-\alpha_0$. Therefore, the confidence that deriving the correct radius is at least $(1-\alpha_0)^2 (1-2e^{-2n\Delta^2})^4 $. In Figure \ref{fig:confidence}, we show the confidence on a varying number of samples when $\Delta=0.1$ and $\alpha_0=0.999$. As the number of samples increases to around $400$ (all our experiments use more than 400 samples), our confidence is very close to Cohen's confidence \cite{cohen2019certified}. Thus, the confidence is nearly 1 in all our experiments. 


\begin{figure}[!h]
    \centering
    \includegraphics[width=60mm]{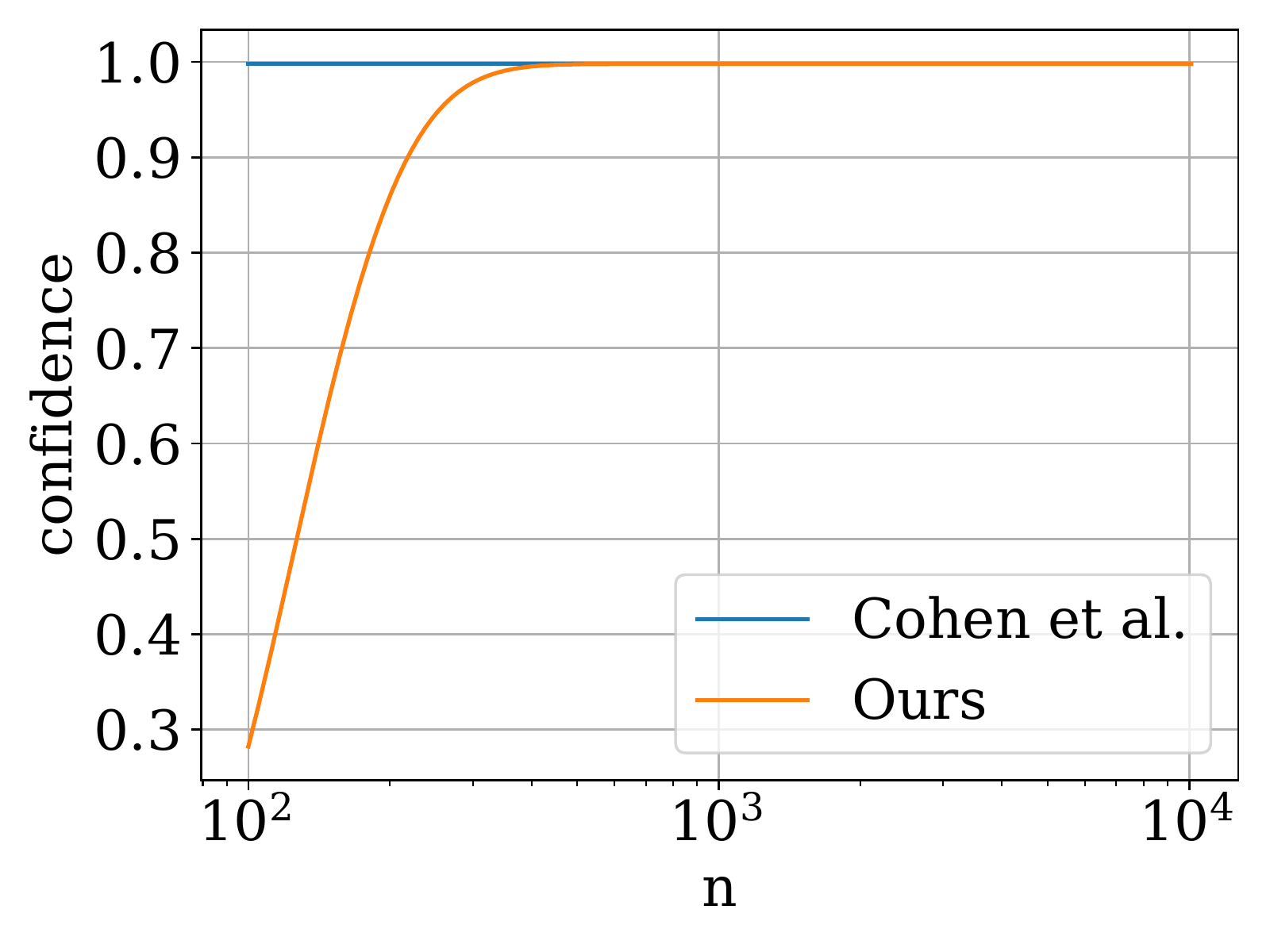}
    \caption{Confidence vs. number of Monte Carlo samples.}
    \label{fig:confidence}
\end{figure}

\subsection{Optimization Convergence}
\label{apd:convergence}
We analysis the convergence of the two-phase optimization and the certification accuracy in this section. On one hand, the optimality of the scalar optimization can be asymptotically achieved by binary search. On the other hand, it is hard to find the minimum $||\lambda \delta||_p$ in the highly-dimensional space, but some special symmetry in the direction of $\delta$ (e.g., spherical symmetry that is also found in \cite{zhang2020black,yang2020randomized}), can help approximate the certified radius. The detailed algorithms are presented in Section \ref{sec:two-phase algorithms}. The defense performance of such universally approximated certified robustness against different real-world attacks is the same as certified robustness (as shown in Appendix \ref{apd: realattack}). Thus, such negligible approximation error is close to 0, but result in many significant new benefits in return.

\section{Algorithms}
\label{apd: algs}
\subsection{Computing $t_A$ and $t_B$}
We present the algorithm to compute the $p_A$ and $p_B$ in Algorithm \ref{alg1}.
\begin{algorithm}
\caption{Computing $t_A$ and $t_B$}
\begin{algorithmic}[1]
\REQUIRE Lower bound of the probabilities, $\underline{p_A}$; upper bound of the probabilities, $\overline{p_B}$; 
perturbation scalar, $\lambda$; 
perturbation, $\delta$; 
noise PDF, $\mu_x$; 
number of samples in the Monte Carlo method, $n$
\ENSURE The auxiliary parameters, $t_A$ and $t_B$
\STATE Sample n noise $\epsilon \in \mathbb{R}^{n \times d}$ from a discrete version of PDF.
\STATE Calculate $\frac{\mu_x(x-\lambda\delta)}{\mu_x(x)}$ using these n samples of noise, $\mu_x$, $\lambda$ and $\delta$
\STATE Estimate the CDF $\Phi$ of $\frac{\mu_x(x-\lambda\delta)}{\mu_x(x)}$ using Monte Carlo method
\RETURN $t_A=\Phi^{-1}(p_A)$ and $t_B=\Phi^{-1}(p_B)$, with inverse CDF $\Phi$
\end{algorithmic}
\label{alg1}
\end{algorithm}

\vspace{-2mm}
\subsection{Scalar Optimization} 
\label{apd: scalarAlg}
We use the binary search to find a scale factor that minimizes $|K|$ (\emph{the distance between $\delta$ and the robustness boundary}). When $K=0$, the perturbation $\delta$ is exactly on the robustness boundary. Fixing the direction of $\delta$, we find two scalars such that $K>0$ and $K<0$. Specifically, we start from a scalar $\lambda_a$ and compute $K$. If $K>0$, then the scaled perturbation $\lambda_a \delta$ is within the robustness boundary, thus we enlarge the scalar to find a $\lambda_b$ such that $K<0$ and vice versa. After that, we iteratively compute the $K$ using $\lambda=\frac{1}{2}(\lambda_a+\lambda_b)$: if $K>0$, we let $\lambda_a=\lambda$; otherwise, we let $\lambda_b=\lambda$. We repeat this iteration until $K$ is less than a threshold or the number of iterations is sufficiently large. The procedures are summarized in Algorithm \ref{alg2}.

\begin{algorithm}
\caption{Scalar Optimization}
\begin{algorithmic}[1]
\REQUIRE Lower bound of the probabilities, $\underline{p_A}$;
upper bound of the probabilities, $\overline{p_B}$;
perturbation scalar, $\lambda$;
perturbation, $\delta$;
noise PDF, $\mu_x$;
number of samples in Monte Carlo method, $n$;
threshold for $K$, $K_m$;
number of iterations for binary search, $N$
\ENSURE The scalar $\lambda$ that minimizes $|K|$

\STATE Find initial scalar $\lambda_a$ and $\lambda_b$ such that $K>0$ and $K<0$
\STATE $\lambda=(\lambda_a+\lambda_b)/2$
\STATE Compute $K$ using $\lambda$
\WHILE{$N>0$ and $|K|>K_m$}
\IF{$K>0$}\STATE $\lambda_a=\lambda$ \ELSE \STATE $\lambda_b=\lambda$\ENDIF
\STATE 
$\lambda=(\lambda_a+\lambda_b)/2$
\STATE 
Compute $K$ using $\lambda$
\STATE N=N-1
\ENDWHILE
\RETURN $\lambda$

\end{algorithmic}
\label{alg2}
\end{algorithm}

\subsection{Direction Optimization}
\label{apd: DirectAlg}

We show how to initialize the positions for different $\ell_p$ norms in PSO. Since some noise follows PDFs with symmetry \cite{zhang2020black,yang2020randomized}, we set the initial position of particles by considering this, e.g., setting the initial positions w.r.t. $\ell_p$ for $p\in \mathbb{R}^+$ as $[0,..., 0,a ,0,...,0]$ and the initial positions w.r.t. $\ell_\infty$ as $[a,a,a,...,a]$, where $a$ is a small random number. Although the search space is highly-dimensional, empirical results show that the radius given by PSO can accurately approximate the theoretical radius given by other methods, e.g., Cohen's \cite{cohen2019certified} (see Figure \ref{fig:P-R curve}). Notice that, for more complicated PDFs without symmetry (which is indeed difficult for deriving the certified radius), PSO can also approximate the certified radius with more particles and iterations. 

\subsection{Hill-climbing algorithm for I-OPT}
\label{apd: I-OPT alg}

The Hill-climbing algorithm is summarized in Algorithm \ref{alg3}.

\begin{algorithm}
\caption{I-OPT with Hill Climbing}
\begin{algorithmic}[1]
\REQUIRE Input data, $x$;
PDF of noise distribution, $\mu_x$;
universally approximated certified robustness, UniCR($\cdot$);
initial hyper-parameters $\bm{\alpha}$;
optimization range of hyper-parameters, $[\bm{L}, \bm{H}]$;
optimization step of hyper-parameters, $\bm{S}$
\ENSURE The optimal hyper-parameters, $\bm{\alpha}_{optimal}$

\STATE Initialize the certified radius $R_0=\text{UniCR}(x, \mu(\bm{\alpha}))$
\STATE For each hyper-parameter $\alpha_i$ in $\bm{\alpha}$:
\IF{$L_i<\alpha_i+S_i<H_i$} \STATE
    $R'=\text{UniCR}(x,\mu(\bm{\alpha}|\alpha_i=\alpha_i+S_i))$ 
    \IF{$R'>R_0$} \STATE $\bm{\alpha}$ is updated with $\alpha_i=\alpha_i+S_i$ \STATE $R_0=R'$
    \ELSIF{$L_i<\alpha_i-S_i<H_i$} \STATE           $R'=\text{UniCR}(x,\mu(\bm{\alpha}|\alpha_i=\alpha_i-S_i))$
    \STATE$R_0=R'$
        \IF {$R'>R$} \STATE $\bm{\alpha}$ is updated with $\alpha_i=\alpha_i-S_i$ \STATE $R_0=R'$ \ENDIF
    \ELSE \STATE break
    \ENDIF
\ELSE \STATE break
\ENDIF
\RETURN $\bm{\alpha}_{optimal}=\bm{\alpha}$

\end{algorithmic}
\label{alg3}
\end{algorithm}

\section{More Experimental Results}
\subsection{Metrics}
\label{apd: Metrics}
We show the illustration of Robustness Score in Figure. \ref{fig:Rscore}

\begin{figure}
    \centering
    \includegraphics[width=55mm]{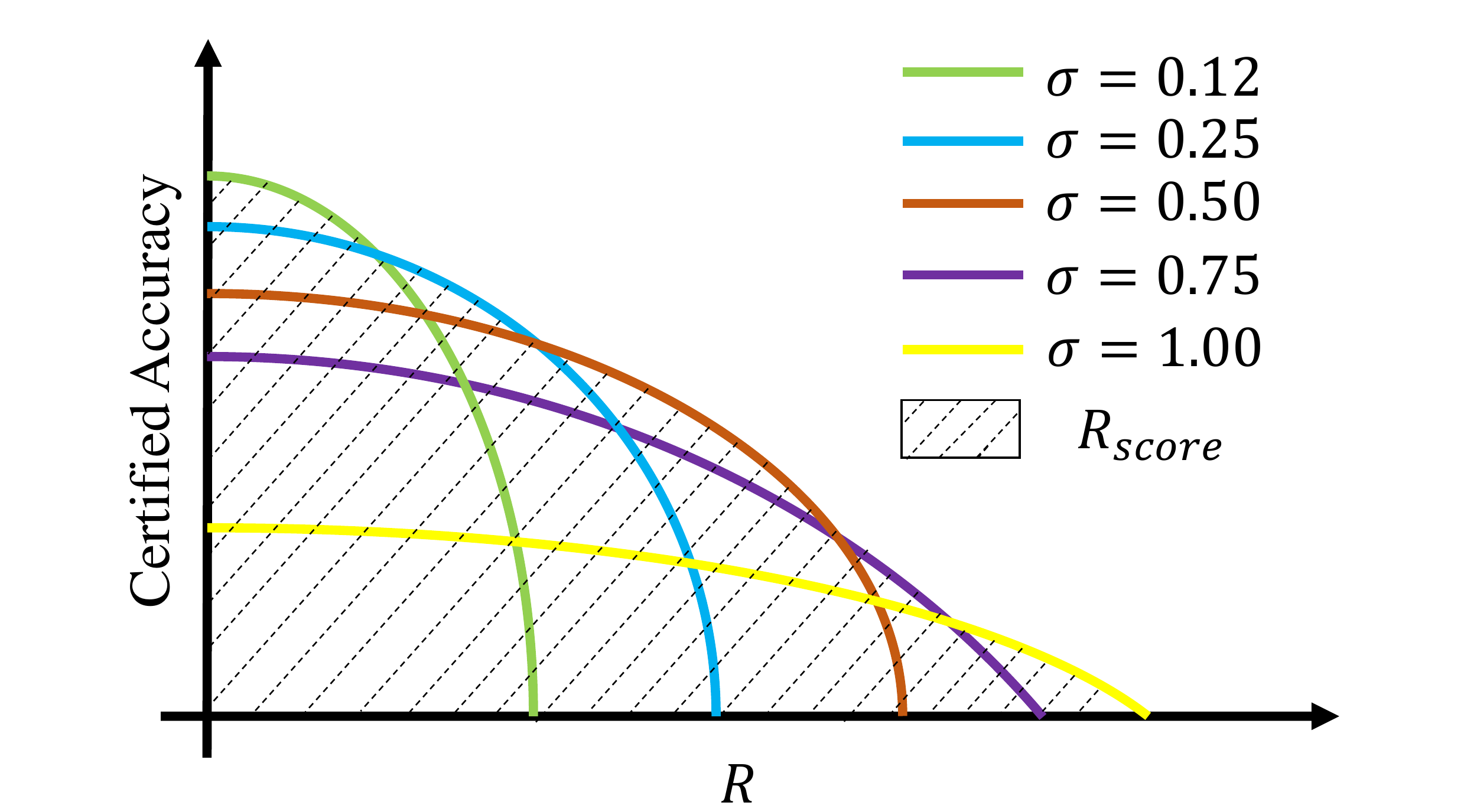}
    \caption{An example of the Robustness Score.}
    \label{fig:Rscore}
\end{figure}

\subsection{Defense against Real Attacks}
\label{apd: realattack}
We evaluate our UniCR's defense accuracy against a diverse set of state-of-the-art attacks, including universal attacks \cite{co2019procedural}, white-box attacks \cite{croce2020reliable,wong2019wasserstein}, and black-box attacks \cite{andriushchenko2020square,chen2020hopskipjumpattack}. We compare UniCR with other state-of-the-art certified schemes \cite{yang2020randomized,cohen2019certified,teng2020ell} against $\ell_1,\ell_2$ and $\ell_\infty$ perturbations. The certified radius $R$ for each image in the test set ($10,000$ images in total) are computed beforehand, and the perturbation generation is constrained by $||\delta||_p=R$ for all the attack methods. 
We define the defense accuracy as the rate that the smoothed classifier can successfully defend against the perturbations with the $\ell_p$ size identical to the the certified radius:



\begin{equation}
    acc_d=\mathbb{E}_{||\delta||_p = R }[\frac{\sum{g(x+\delta)=c_A}}{N}]
\end{equation} where $c_A=g(x)$, $N$ is the total test number. In this defense study, we use 
$500$ samples for both Monte Carlo method and testing.


Table \ref{tab:real attack cifar10} shows the defense accuracy on the smoothed classifier. The attack with "$*$" 
is re-scaled to the required norm (perturbation size $R$) based on their perturbation formats. UniCR 
universally provides a 100\% defense accuracy against all the $\ell_1$, $\ell_2$ and $\ell_\infty$ perturbations generated by all the state-of-the-art attacks. These results validate 
our universally approximated certified robustness ensures the same defense 
performance as certified robustness in practice. 

\begin{table*}[!h]
\vspace{-2mm}
    \centering
    \normalsize
    \caption{Defense against real attacks on CIFAR10 (results on MNIST \& ImageNet are similar and not included due to space limit).}\vspace{-0.1in}
    \resizebox{\linewidth}{!}{
    \begin{tabular}{c c c c c c c}
    \hline

        Defense Accuracy ($\%$)   & Gaussian*     &  Procedural*  \cite{co2019procedural} & Auto-PGD \cite{croce2020reliable} & Wasserstein* \cite{wong2019wasserstein} & Square* \cite{andriushchenko2020square} & HSJ* \cite{chen2020hopskipjumpattack}  \\ \hline
    
      Teng's \cite{teng2020ell} $\ell_1$-norm R                             & 100.00    & 100.00    & 100.00    &100.00     &100.00     &100.00  \\
        Our $\ell_1$-norm R                                                 & 100.00    & 100.00    & 100.00    &100.00     &100.00     &100.00  \\
        Cohen's \cite{cohen2019certified} $\ell_2$-norm R                   & 100.00    & 100.00    & 100.00    &100.00     &100.00     &100.00  \\
        Our $\ell_2$-norm R                                                 & 100.00    & 100.00    & 100.00    &100.00     &100.00     &100.00  \\
        Yang's \cite{yang2020randomized} $\ell_\infty$-norm R               & 100.00    & 100.00    & 100.00    &100.00     &100.00     &100.00  \\
        Our $\ell_\infty$-norm R                                            & 100.00    & 100.00    & 100.00    &100.00     &100.00     &100.00  \\
    \hline
    
      
    \end{tabular}
    }
    \label{tab:real attack cifar10}
\end{table*}
\subsection{List of PDFs}
The PDFs used in our experimental are summarized in Table \ref{tab:pdfs}.
\label{apd: pdflist}
\begin{table}[!h]
    \centering
    \scriptsize
    \caption{List of noise distributions.}
    \begin{tabular}{p{6cm} p{6cm}}
    \hline
        Distribution & Probability Density Function\\
    \hline
         Gaussian & $\propto e^{-|x/\alpha|^2}$ \\
         Laplace & $\propto e^{-|x/\alpha|}$\\
         Hyperbolic Secant & $\propto sech(|x/\alpha|)$\\
         General Normal & $\propto e^{-|x/\alpha|^\beta}$\\
         Cauthy & $\propto \frac{\alpha^2}{x^2+\alpha^2}$\\
         Pareto & $\propto \frac{1}{(1+|x/\alpha|)^{\beta+1}}$\\
         Laplace-Gaussian Mix. & $\propto \beta e^{-|x/\alpha|^1}+ (1-\beta) e^{-|x/\alpha|^2}$\\
         Exponential Mix. & $\propto e^{-\beta |x/\alpha|^1-(1-\beta) |x/\alpha|^2}$\\
    \hline
    \end{tabular}
    \label{tab:pdfs}
\end{table}

\subsection{Efficiency for Radius Derivation}
\label{sec:eff}

We show the runime of our algorithms on deriving the certified radius for the inputs with various input dimensions in Figure \ref{fig:runtime}. For the common input dimensions, e.g., $24\times 24$ for MNIST, $3 \times 32 \times 32$ for CIFAR10, and $3\times 224 \times 224$ for ImageNet, it takes less than 10 seconds for certifying an image on average.   Comparing with the theoretical certified radius deriving, our method's running time is undoubtedly larger since their radius is pre-derived. However, with the significant benefits on the universality and the automatically deriving, we believe the cost of the extra running time is worthwhile and acceptable in practice.

\begin{figure}[!h]
\vspace{-4mm}
    \centering
    \includegraphics[width=70mm]{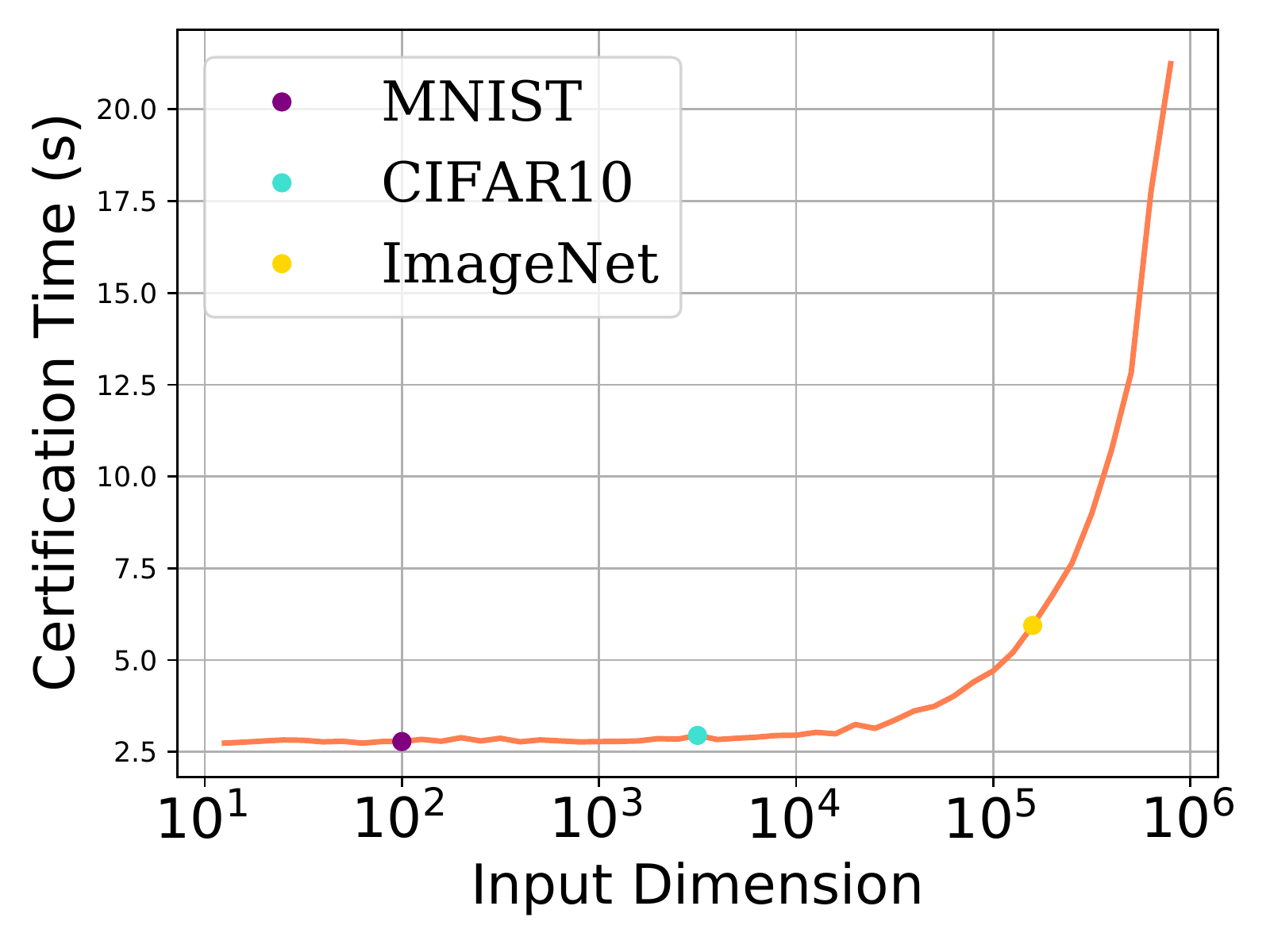}
    \vspace{-0.15in}
    \caption{Runtime of UniCR vs. input sizes (with RTX3080 GPU).}
    \label{fig:runtime}
\end{figure}

\subsection{Any $p$ (besides 1, 2, $\infty$)}
\label{sec: dif p}

\begin{figure}
    \centering
    \includegraphics[width=70mm]{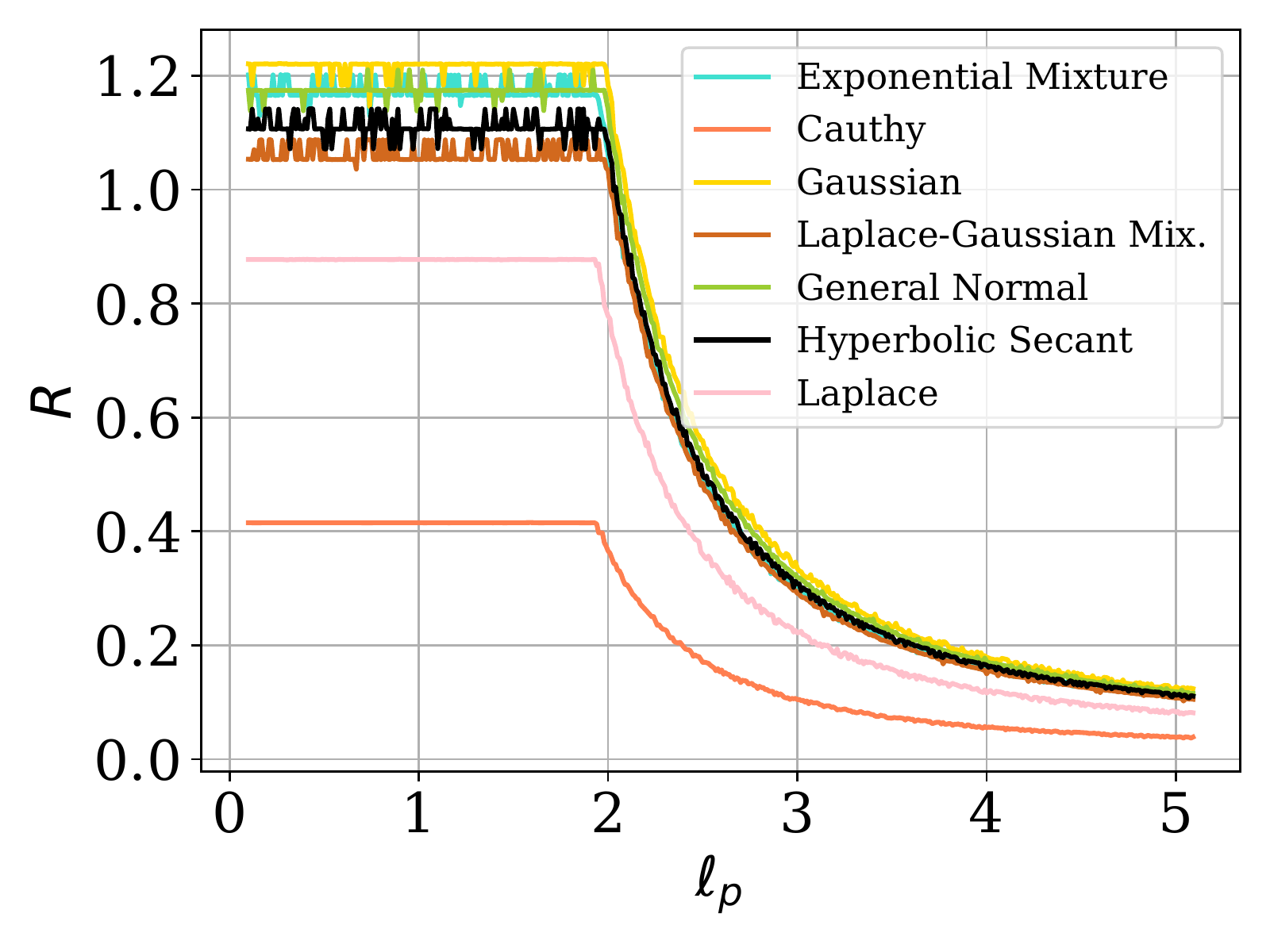}
    \caption{Radius vs. various $\ell_p$ pert.}
    \label{fig:dif p}
    
\end{figure}

Existing methods \cite{cohen2019certified,teng2020ell} usually focus on the certified radius in a specific norm, e.g., $\ell_1$, $\ell_2$ or $\ell_\infty$ norms. Some methods \cite{zhang2020black,yang2020randomized} provide certified robustness theories for multiple norms but specific settings are usually needed for deriving the certified radii in different norms. None of the existing methods can automatically compute the certified radius in any $\ell_p$ norm. In this section, we show our UniCR can automatically approximate the certified radii for various $p$, in which $p$ is a real number greater than 0. 

In the experiments, we set the probability $\underline{p_A}=0.9$ and draw the lines of certified radius w.r.t. different $p$ for $p>0$. We show the results computed with different noise distributions in Figure \ref{fig:dif p}. We observe that when $p \in (0,2]$, the certified radius for different $p$ are approximately identical. This finding also matches the theoretical results in Yang et al. \cite{yang2020randomized}, in which the certified radii in $\ell_1$ and $\ell_2$ norm are exactly the same for multiple distributions. When $p>2$, we observe that the certified radius decreases as $p$ increases.

\subsection{Evaluations on Complicated PDFs}
\label{apd:fine-grain evaluation}

\begin{figure*}[t]
    \centering
    \subfigure[General Normal vs. $\ell_1$ perturbations]{
    \begin{minipage}[t]{0.3\linewidth}
    \centering
    \includegraphics[width=38 mm]{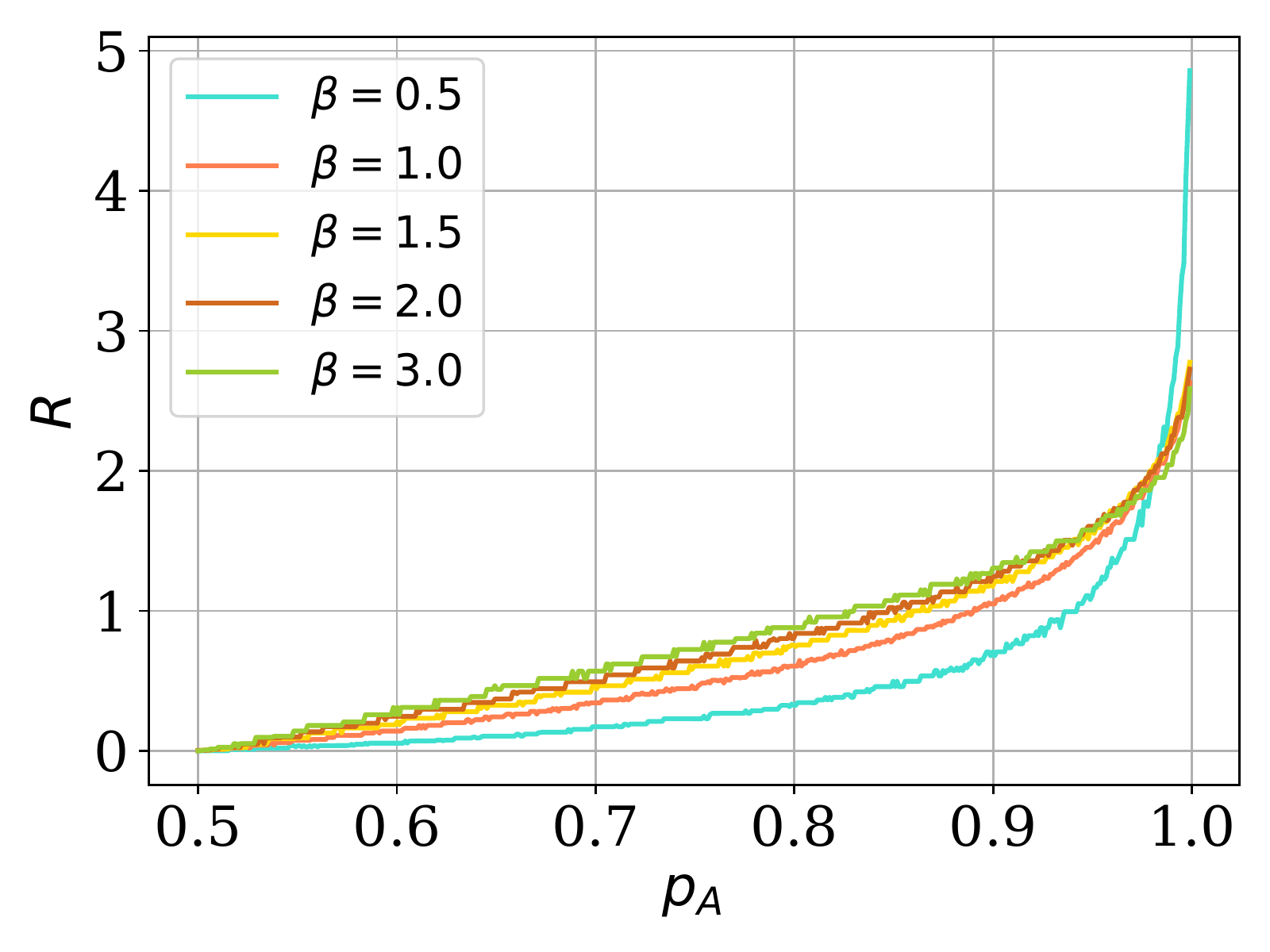}
    \end{minipage}
    }
    \subfigure[General Normal vs. $\ell_2$ perturbations]{
    \begin{minipage}[t]{0.3\linewidth}
    \centering
    \includegraphics[width=38 mm]{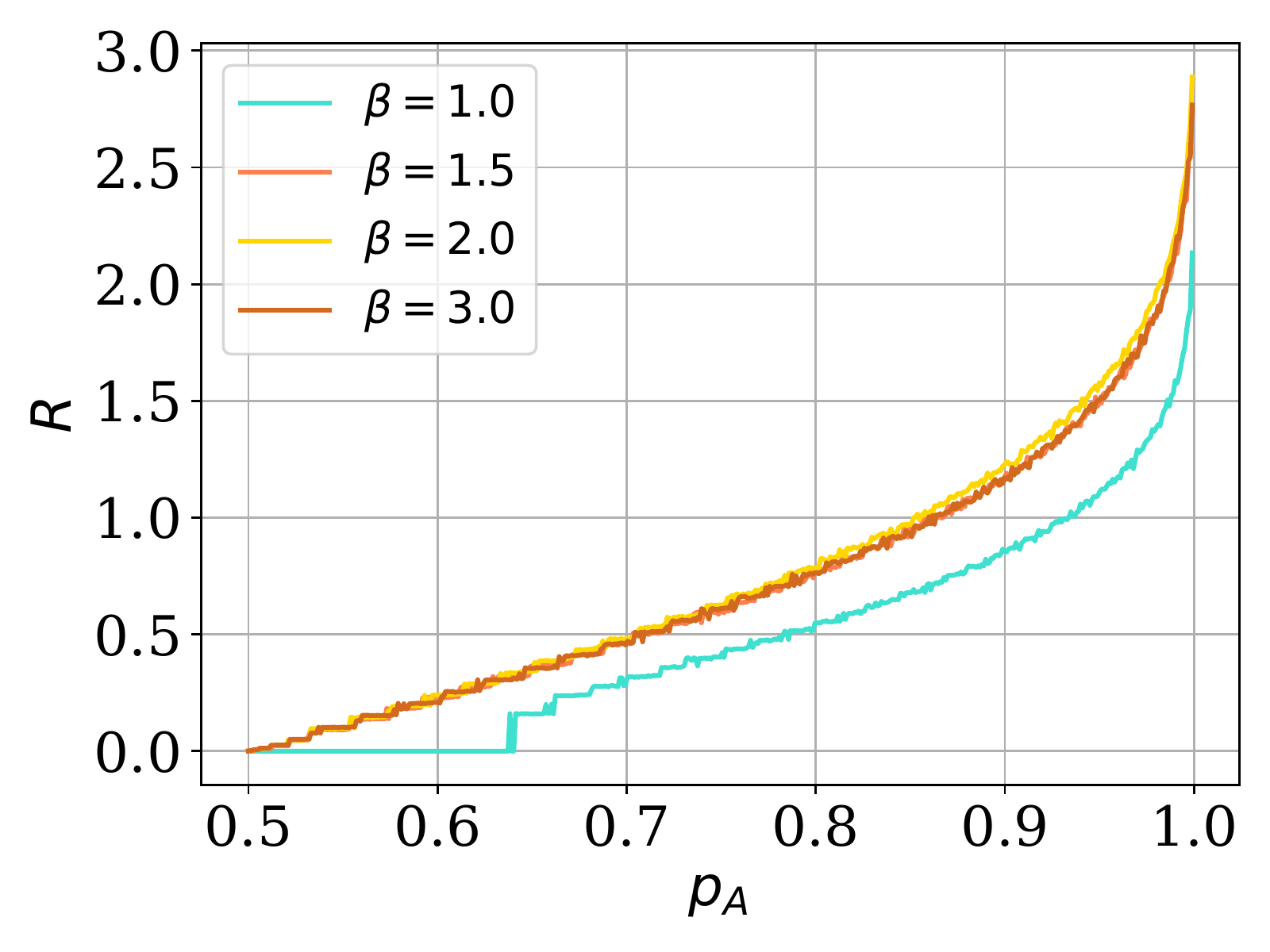}
    \end{minipage}
    }
    \subfigure[General Normal vs. $\ell_\infty$ perturbations]{
    \begin{minipage}[t]{0.3\linewidth}
    \centering
    \includegraphics[width=38 mm]{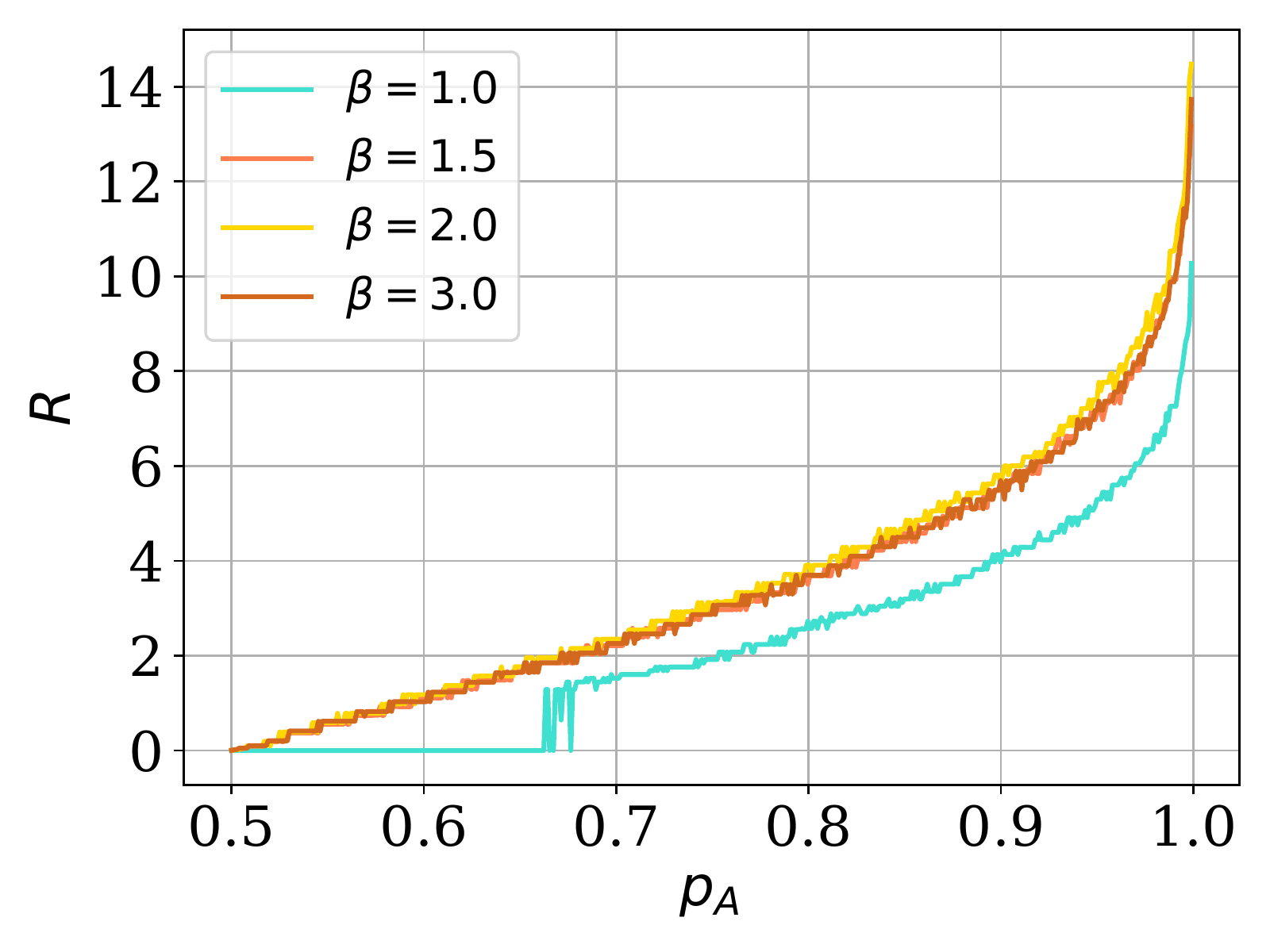}
    \end{minipage}
    }
    \subfigure[Laplace-Gaussian Mixture vs. $\ell_\infty$ perturbations]{
    \begin{minipage}[t]{0.3\linewidth}
    \centering
    \includegraphics[width=38 mm]{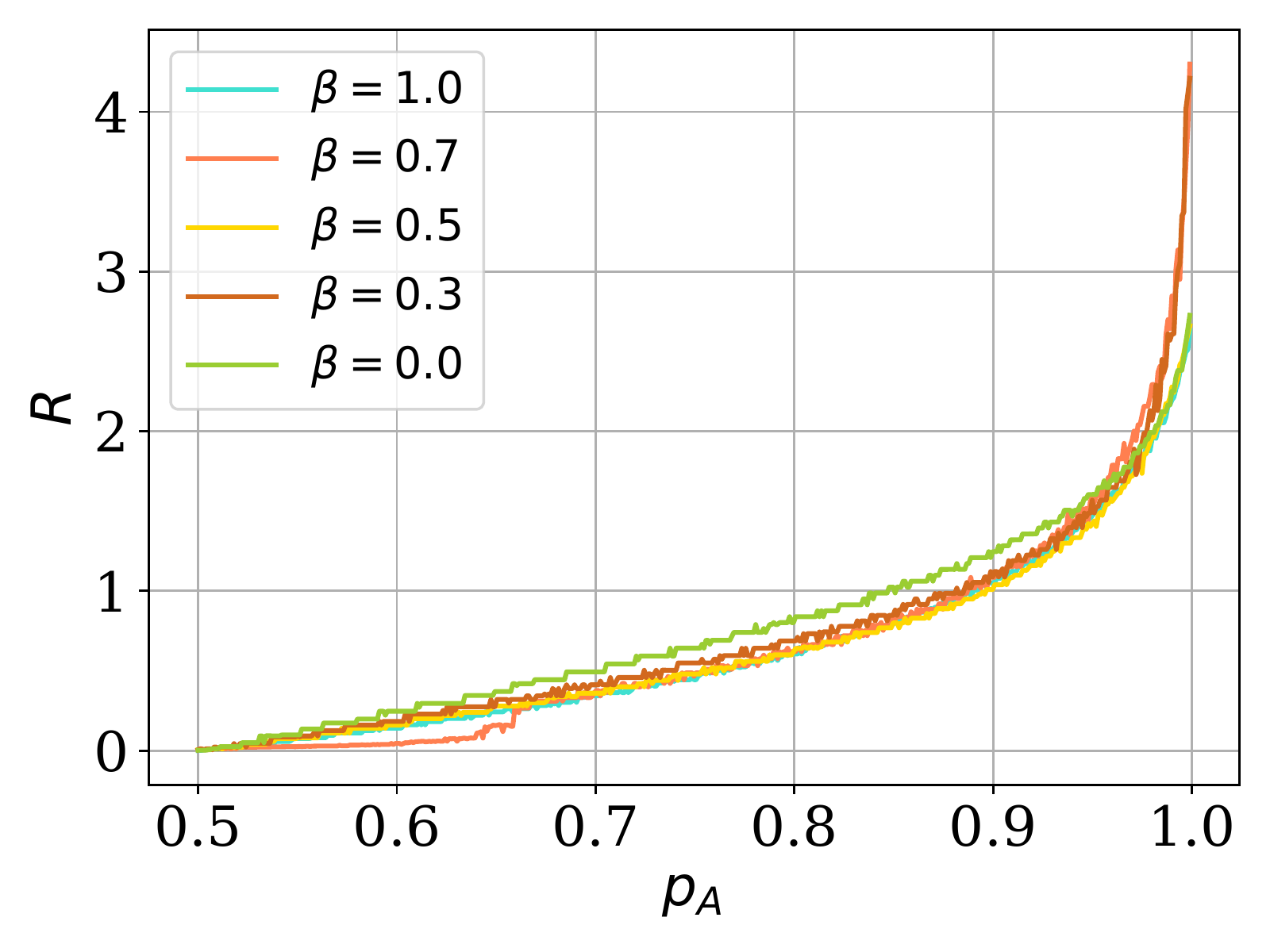}
    \end{minipage}
    }
    \subfigure[Laplace-Gaussian Mixture vs. $\ell_\infty$ perturbations]{
    \begin{minipage}[t]{0.3\linewidth}
    \centering
    \includegraphics[width=38 mm]{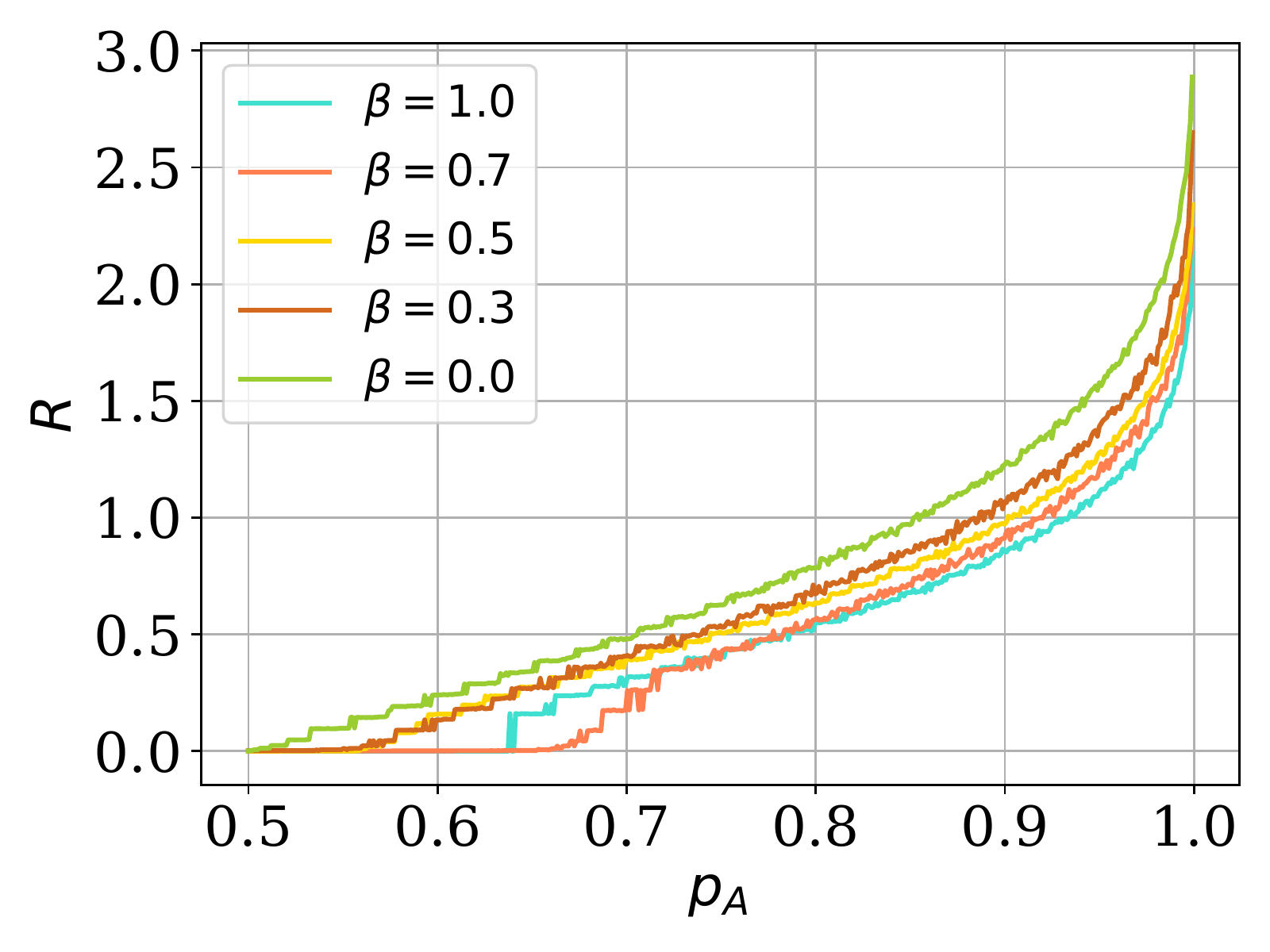}
    \end{minipage}
    }
    \subfigure[Laplace-Gaussian Mixture vs. $\ell_\infty$ perturbations]{
    \begin{minipage}[t]{0.3\linewidth}
    \centering
    \includegraphics[width=38 mm]{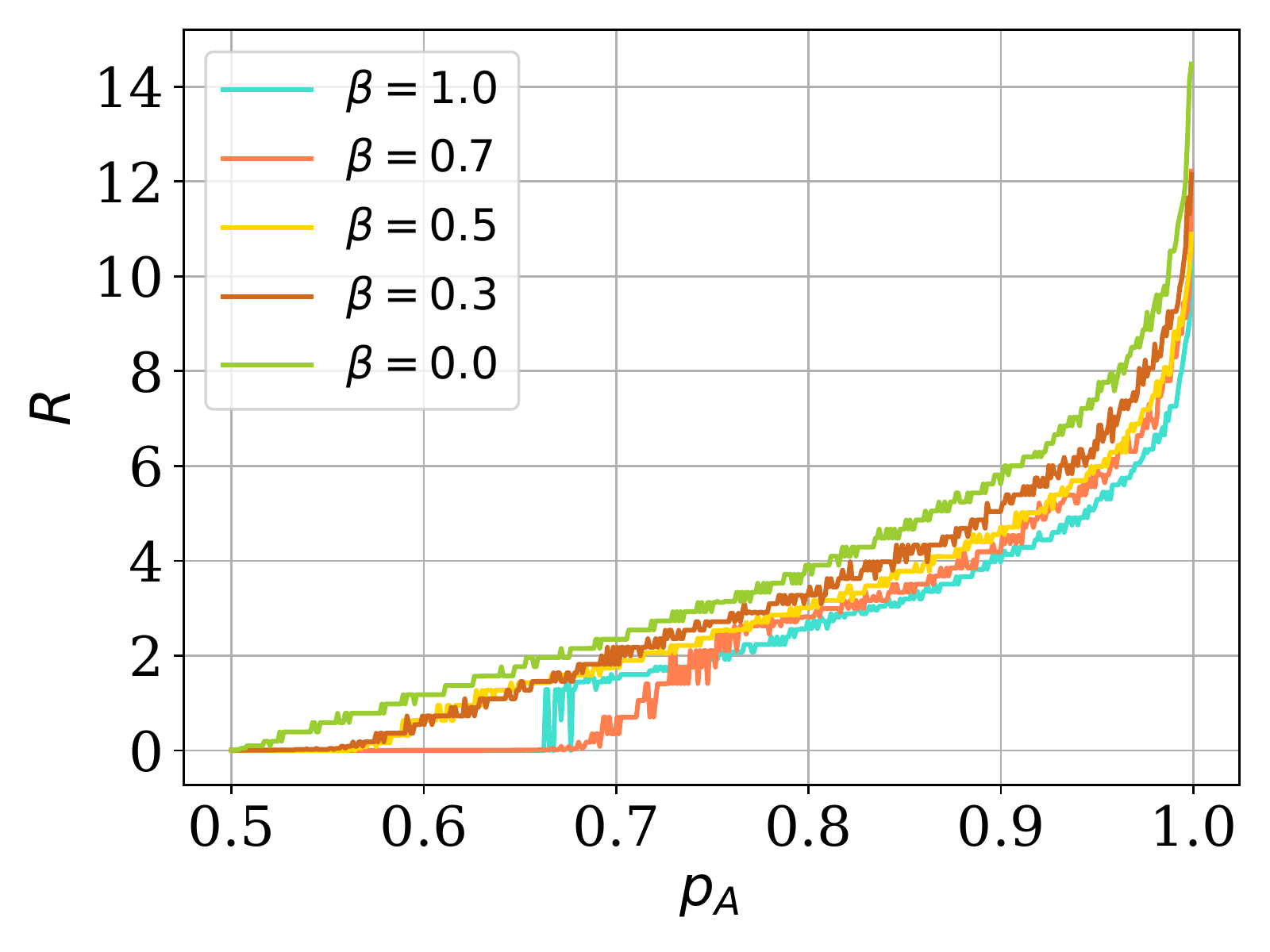}
    \end{minipage}
    }
    
    \subfigure[Exponential Mixture vs. $\ell_\infty$ perturbations]{
    \begin{minipage}[t]{0.3\linewidth}
    \centering
    \includegraphics[width=38 mm]{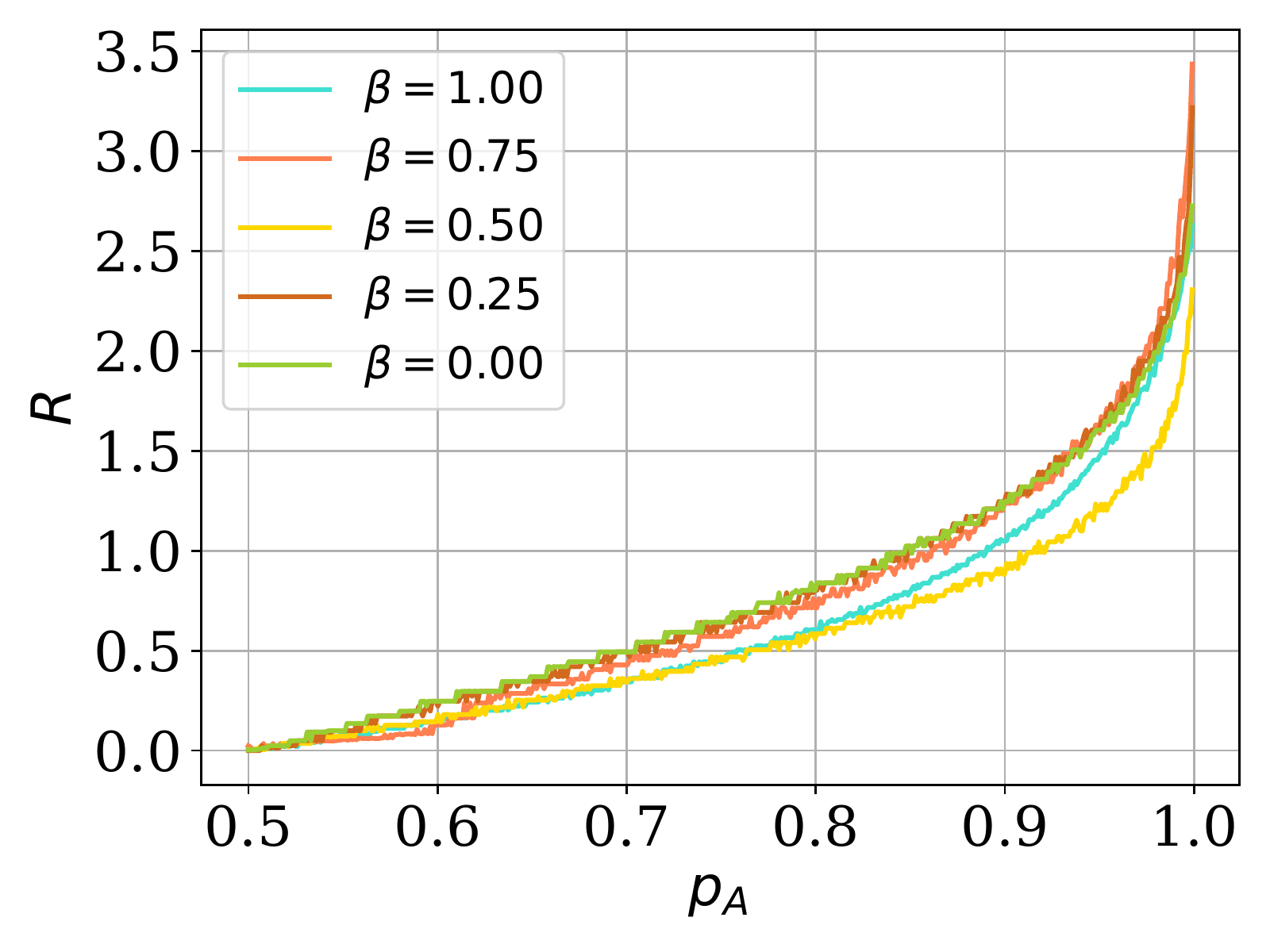}
    \end{minipage}
    }
    \subfigure[Exponential Mixture vs. $\ell_\infty$ perturbations]{
    \begin{minipage}[t]{0.3\linewidth}
    \centering
    \includegraphics[width=38 mm]{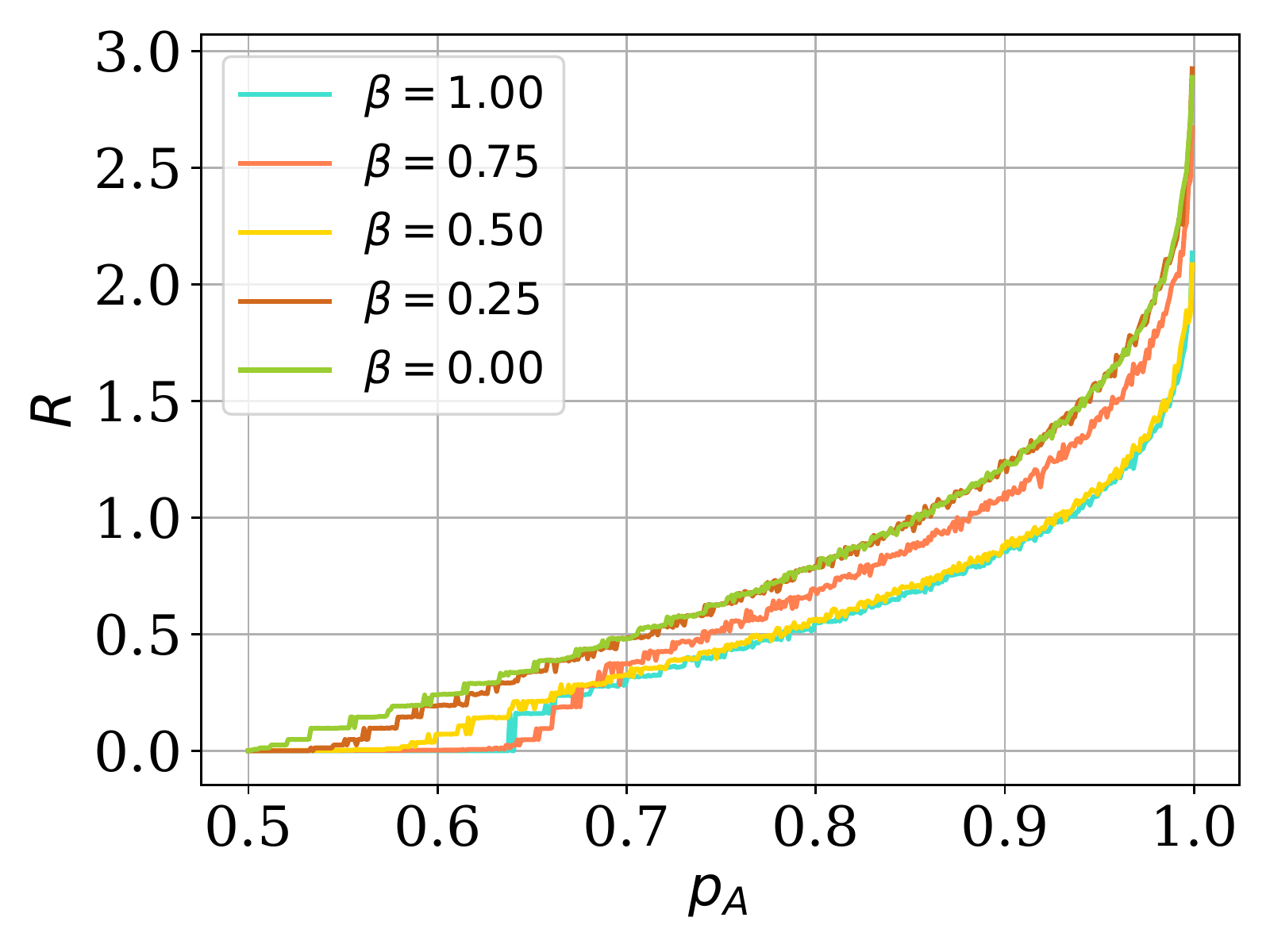}
    \end{minipage}
    }
    \subfigure[Exponential Mixture vs. $\ell_\infty$ perturbations]{
    \begin{minipage}[t]{0.3\linewidth}
    \centering
    \includegraphics[width=38 mm]{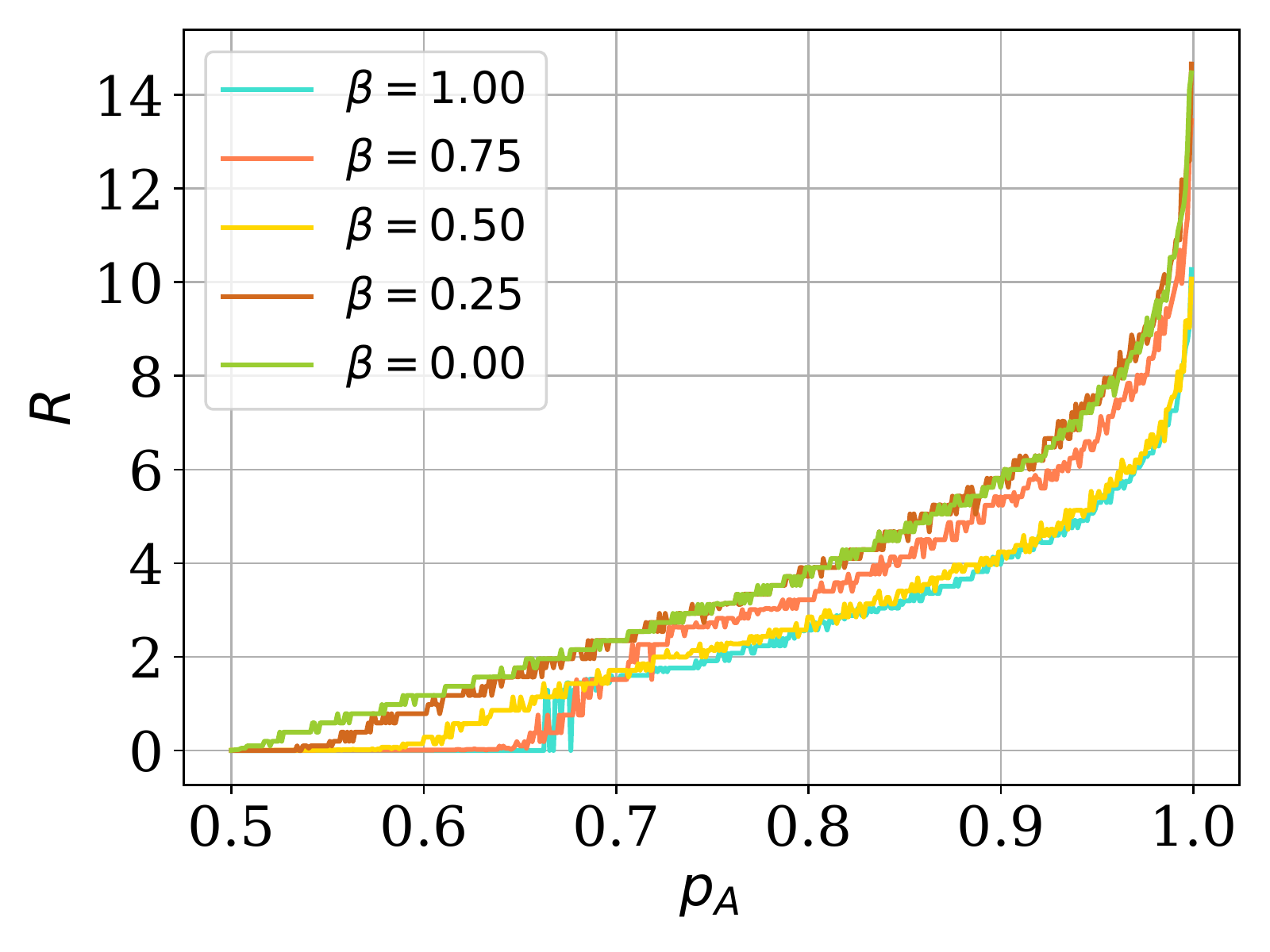}
    \end{minipage}
    }
    
    \centering
    \caption{$p_A$-$R$ curves of General Normal, Laplace-Gaussian Mixture, and Exponential Mixture noise with a varying $\beta$. }
    \label{fig:P-R galary}
\end{figure*}

We provide a fine-grained evaluation on the complicated distributions \cite{mohammady2020r2dp}, e.g., General Normal, Laplace-Gaussian Mixture, and Exponential Mixture noises with various $\beta$. It shows that the Gaussian (i.e., $\beta=2$ for General Normal, $\beta=0$ for Laplace-Gaussian Mixture and Exponential Mixture) is the optimal noise in these $\beta$ setting. We also observe the ``crash'' on Laplace-based distributions when $p_A$ is small.

\subsection{Certification on Non-Smoothed Classifier}

\begin{table*}[!h]
    \centering
    \normalsize
    \caption{Certified accuracy on standard classifier.}
    \resizebox{\linewidth}{!}{
    \begin{tabular}{c c c c c c c c c c}
    \hline
         radius $R$ & 0.0 & 0.1 & 0.2 & 0.3 & 0.4 & 0.5 &0.6 & 0.7 &0.8 \\
    \hline
         Yang's \cite{yang2020randomized} vs. $\ell_1$-norm  & 10.6 & 10.4 & 10.4 &9.8 &8.8 & 8.2 &5.4 &2.2 &1.0 \\
         Ours vs. $\ell_1$-norm & \textbf{98.8} & \textbf{47.0} & \textbf{22.4 }&\textbf{17.8} &\textbf{13.8} & \textbf{10.2} &\textbf{7.0} &\textbf{3.8} &1.0   \\
         \hline
         Cohen's  \cite{cohen2019certified} vs. $\ell_2$-norm  &10.6 &10.4 &10.4 &9.6 &8.8 &8.2 &5.6 &2.2 &1.2 \\
         Ours vs. $\ell_2$-norm  &\textbf{98.8} &\textbf{46.0} &\textbf{22.4} &\textbf{17.6} &\textbf{13.8} &\textbf{9.8} &\textbf{7.0} &\textbf{3.8} &1.2 \\
         \hline
         Yang's \cite{yang2020randomized} vs. $\ell_\infty$-norm (at $R/255$) &10.6 &10.6 &10.6 &10.4 &10.4 &10.4 &10.4 &10.4 &10.4 \\
         Ours vs. $\ell_\infty$-norm (at $R/255$) &\textbf{98.6} &\textbf{92.4} &\textbf{69.4} &\textbf{61.6} &\textbf{53.6} &\textbf{46.0} &\textbf{37.8} &\textbf{27.4} &\textbf{24.4} \\
    \hline
    \end{tabular}}
    \label{tab:std classifier}
\end{table*}

Besides certifying inputs with the smoothed classifier, our input noise optimization (I-OPT) can certify input with a standard classifier without degrading the classifier accuracy on clean data (on the contrary, existing works have to trade off such accuracy for certified defenses). 


Specifically, since our I-OPT allows the noise for the input certification to be different from the noise used in training, a special case of the training noise is no noise ($\sigma=0$). This means that we can certify a naturally-trained classifier (standard classifier). This provides an obvious benefit that the classifier can still execute normal classification on clean data with high accuracy since the standard classifier is trained without noise. Also, with I-OPT, we can tune the noise for the input to maintain the prediction accuracy. Thus, any classifier can be certifiably protected against perturbations without degrading the general performance on clean data. 

To maintain the performance on standard classification, we add a condition while performing I-OPT:

\begin{equation}
    g(x+\delta)=f(x)
\end{equation}

We show this application on a standard ResNet110 classifier trained on CIFAR10 (see Table \ref{tab:std classifier}). For the baselines, we use Gaussian noise ($\sigma=0.35$) and its corresponding theoretical radius \cite{yang2020randomized,cohen2019certified} for certification. Our method uses I-OPT with General Normal noise and initializes it with the same $\sigma$. While approximating the certified radius with UniCR, we generate $4,000$ samples with the Monte Carlo method on CIFAR10. 

The table shows that over $98.6\%$ of the inputs are certified by our method with a radius $R>0$. This means that over $98.6\%$ of the samples are certifiably protected while only $10.6\%$ of inputs are certified by the baselines, which is nearly the accuracy by random guessing. This significant improvement emerges since the I-OPT could optimize the noise PDF for each input even though the classifier is not trained with noise (non-smoothed classifier). Although the certified radii are low compared to smoothly-trained classifiers, it provides a certifiable protection on perturbed data while maintaining the high accuracy for classifying clean data.

\section{Visual Examples of I-OPT}
\label{apd:examples}

We present some examples of I-OPT on the ImageNet dataset against $\ell_1$, $\ell_2$ and $\ell_\infty$ perturbations, respectively (see Figure \ref{fig:an example}). In the first case ($\ell_1$ perturbations), without executing the I-OPT, our UniCR certifies the input with a radius $R=1.24$. Our I-OPT optimizes the distribution as the right-most figure shows, then the certified radius is improved to $1.48$ with our UniCR. Similarly, in the rest cases, we show I-OPT can improve the certified radius significantly by optimizing the noise distribution. Especially, we improve the radius from $0.35$ to $1.30$ in the second case.

\begin{figure*}[!h]
    \centering
    \includegraphics[width=\linewidth]{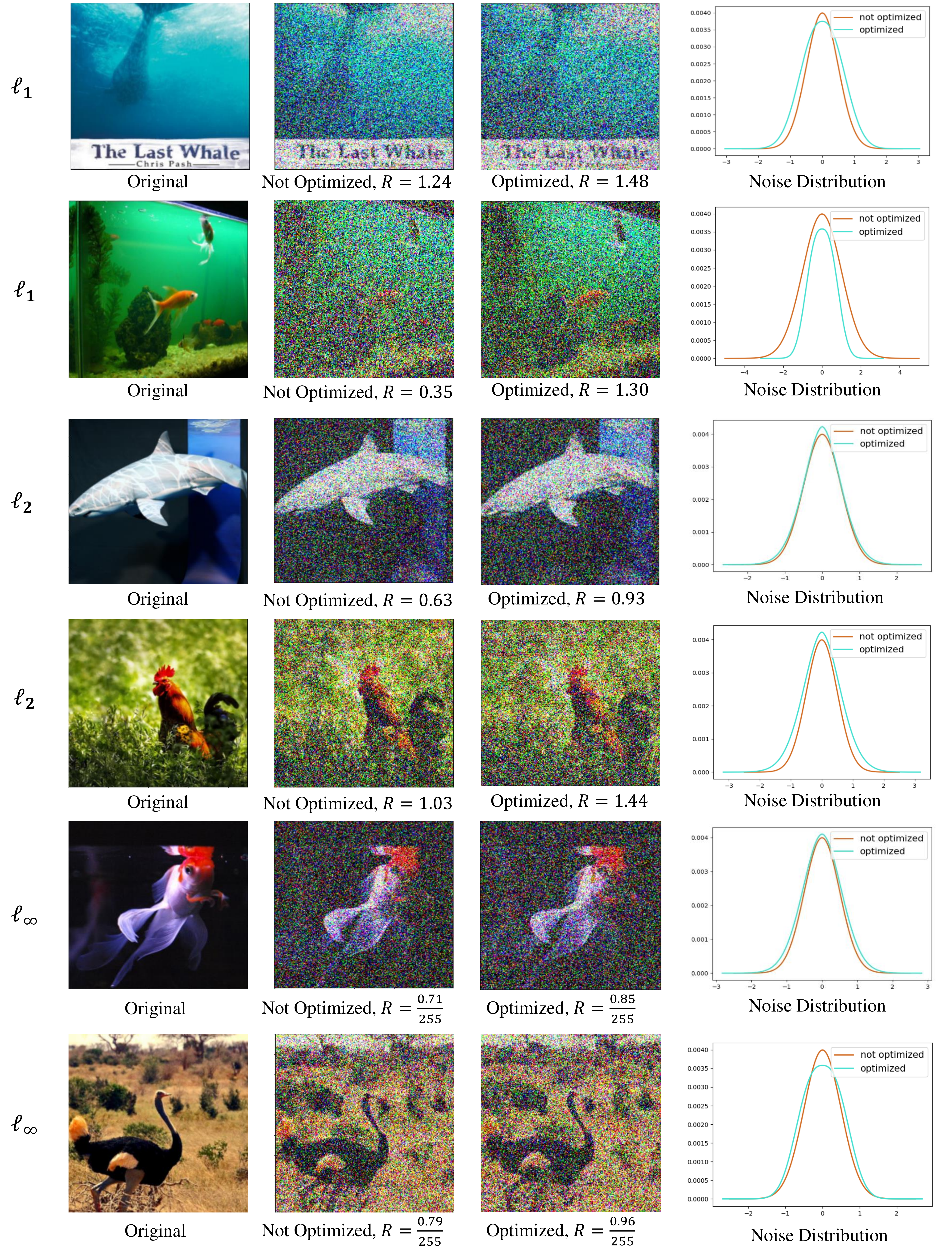}
    \caption{\textmd{Example images of applying I-OPT (based on UniCR) for smoothed classifier against different $\ell_p$ perturbations on the ImageNet dataset. From the left to right, the first figure shows the original image. The second and third figures show the smoothed image without I-OPT and with I-OPT, respectively. The fourth figure shows the corresponding distributions before/after I-OPT.}}
    \label{fig:an example}
\end{figure*}

\section{Discussions}
\label{sec:dis}

\noindent\textbf{Universal Certified Robustness}. 
It might be impractical to make a universal framework satisfy all the theoretical conditions w.r.t. all $\ell_p$ perturbations, especially $p$ can be any positive real number. Thus, we admit that UniCR may not strictly satisfy certified robustness all the time due to the approximated optimization. However, extensive empirical results confirm that our derived radii highly approximate the theoretical certified radii against different $\ell_p$ perturbations. In addition, the defense performance against real attacks also illustrate that our method is as reliable as different theoretical certified radii. We believe that with the negligible error in practice, UniCR can be deployed as a universal framework to significantly ease the process of achieving certified robustness in different scenarios.






\vspace{0.05in}

\noindent\textbf{Certifying Perturbed Data with Randomized Smoothing}. Traditional randomized smoothing usually assumes that the input is clean and empirical defenses \cite{madry2018towards,hong2022eye} are not applied, if the input data is perturbed before certification, then certification in I-OPT might be inaccurate. Indeed, the certification in traditional randomized smoothing (e.g., \cite{cohen2019certified}) methods also depend on the inputs (since $p_A$ is different for different inputs), they might be inaccurate if the input data is perturbed, either. Thus, randomized smoothing based approaches focus on certifying clean inputs rather than correcting perturbed inputs. We will study this interesting problem on certifying both clean and perturbed inputs in the future.




\vspace{+2mm}

\noindent\textbf{Can existing methods adopt noise optimization?} A question here is that if the noise optimization can improve the certified radius, can the theoretical methods provide personalized randomization for each input? The personalized randomization is actually not adaptable in the theoretical methods since they cannot automatically derive the certified radius for different noise distributions, especially for uncommon distributions, e.g., $e^{-|x/0.5|^{1.5}}$. Instead, our UniCR can automatically derive the certified radius for any distribution within the continuous parameter space. 
\vspace{+2mm}

\noindent\textbf{Extensions}. We evaluate our UniCR on the image classification. Indeed, our UniCR is a general method that can be directly applied to other tasks, e.g., video classification \cite{DBLP:conf/ndss/LiNPSKRS19,xie2022universal}, graph learning (e.g., node/graph classification \cite{DBLP:conf/kdd/WangJCG21} and community detection~\cite{jia2020certified}), and natural language processing \cite{DBLP:conf/emnlp/JiaRGL19}.

\end{document}